\documentclass{article}

\usepackage{microtype}
\usepackage{graphicx}
\usepackage{booktabs} 

\usepackage{hyperref}
\hypersetup{colorlinks=true,linkcolor=blue,citecolor=blue}
\usepackage{amsmath}
\usepackage[capitalise]{cleveref}

\usepackage[accepted]{icml2020}

\usepackage{amsfonts}
\usepackage[shortlabels]{enumitem}
\setlist[itemize]{leftmargin=20pt,topsep=-2pt,itemsep=0.7ex,partopsep=5pt,parsep=1pt}
\setlist[enumerate]{leftmargin=20pt,topsep=-2pt,itemsep=0.7ex,partopsep=5pt,parsep=1pt}
\usepackage{algorithm}
\usepackage{macros}
\usepackage{todonotes}
\usepackage{subcaption}
\usepackage[export]{adjustbox}
\usepackage{expl3}[2012-07-08]
\usepackage{zi4}  
\ExplSyntaxOn
\cs_new_eq:NN \fpeval \fp_eval:n
\ExplSyntaxOff

\newcommand{\suppmat}{appendix\@\xspace}

\newcommand{\Suppmatsection}{Appendix}

\definecolor{amaranth}{rgb}{0.9, 0.17, 0.31}
\definecolor{awesome}{rgb}{1.0, 0.13, 0.32}

\icmltitlerunning{
    Understanding Contrastive Representation Learning through Alignment and Uniformity on the Hypersphere
}

\begin{document}

\twocolumn[
\icmltitle{
    Understanding Contrastive Representation Learning through\texorpdfstring{\\}{ }Alignment and Uniformity on the Hypersphere 
}



\icmlsetsymbol{equal}{*}

\begin{icmlauthorlist}
\icmlauthor{Tongzhou Wang}{mit}
\icmlauthor{Phillip Isola}{mit}
\end{icmlauthorlist}

\icmlaffiliation{mit}{MIT Computer Science \& Artificial Intelligence Lab (CSAIL)}

\icmlcorrespondingauthor{Tongzhou Wang}{tongzhou@mit.edu}

\icmlkeywords{Machine Learning, ICML}

\vskip 0.3in
]



\printAffiliationsAndNotice{}  

\begin{abstract}

Contrastive representation learning has been outstandingly successful in practice.
In this work, we identify two key properties related to the contrastive loss: (1) \emph{alignment} (closeness) of features from positive pairs, and (2) \emph{uniformity} of the induced distribution of the (normalized) features on the hypersphere.
We prove that, asymptotically, the contrastive loss optimizes these properties, and analyze their positive effects on downstream tasks.
Empirically, we introduce an optimizable metric to quantify each property.
Extensive experiments on standard vision and language datasets confirm the strong agreement between \emph{both} metrics and downstream task performance.
Directly optimizing for these two metrics leads to representations with comparable or better performance at downstream tasks than contrastive learning.

\vspace{-0.75pt}\noindent\begin{tabular}{@{}lr@{}}
Project Page: & \href{https://ssnl.github.io/hypersphere}{\small\texttt{ssnl.github.io/hypersphere}}.\\
Code: & \href{https://github.com/SsnL/align_uniform}{\small\texttt{github.com/SsnL/align\_uniform}}. \\[-0.25pt]
 & \hspace{-2.05em}\href{https://github.com/SsnL/moco_align_uniform}{\small\texttt{github.com/SsnL/moco\_align\_uniform}}.
\end{tabular}\vspace{-9.75pt}

\end{abstract}

\section{Introduction}

A vast number of recent empirical works learn representations with a unit $\ell_2$ norm constraint, effectively restricting the output space to the unit hypersphere \citep{parkhi2015deep,schroff2015facenet,liu2017sphereface,hasnat2017mises,wang2017normface,bojanowski2017unsupervised,mettes2019hyperspherical,hou2019learning,s-vae18,xu2018spherical}, including many unsupervised contrastive representation learning methods \citep{wu2018unsupervised,bachman2019learning,tian2019contrastive,he2019momentum,chen2020simple}.

Intuitively, having the features live on the unit hypersphere leads to several desirable traits.
Fixed-norm vectors are known to improve training stability in modern machine learning where dot products are ubiquitous \citep{xu2018spherical,wang2017normface}. Moreover, if features of a class are sufficiently well clustered, they are linearly separable with the rest of feature space (see Figure~\ref{fig:hypersphere_linsep}), a common criterion used to evaluate representation quality.

\begin{figure}[t]
    \centering\vspace{-4.5pt}
    \begin{subfigure}[b]{\linewidth}
    \centering
    \captionsetup{justification=centering}
    \includegraphics[width=0.88\linewidth, trim=500 180 500 200, clip]{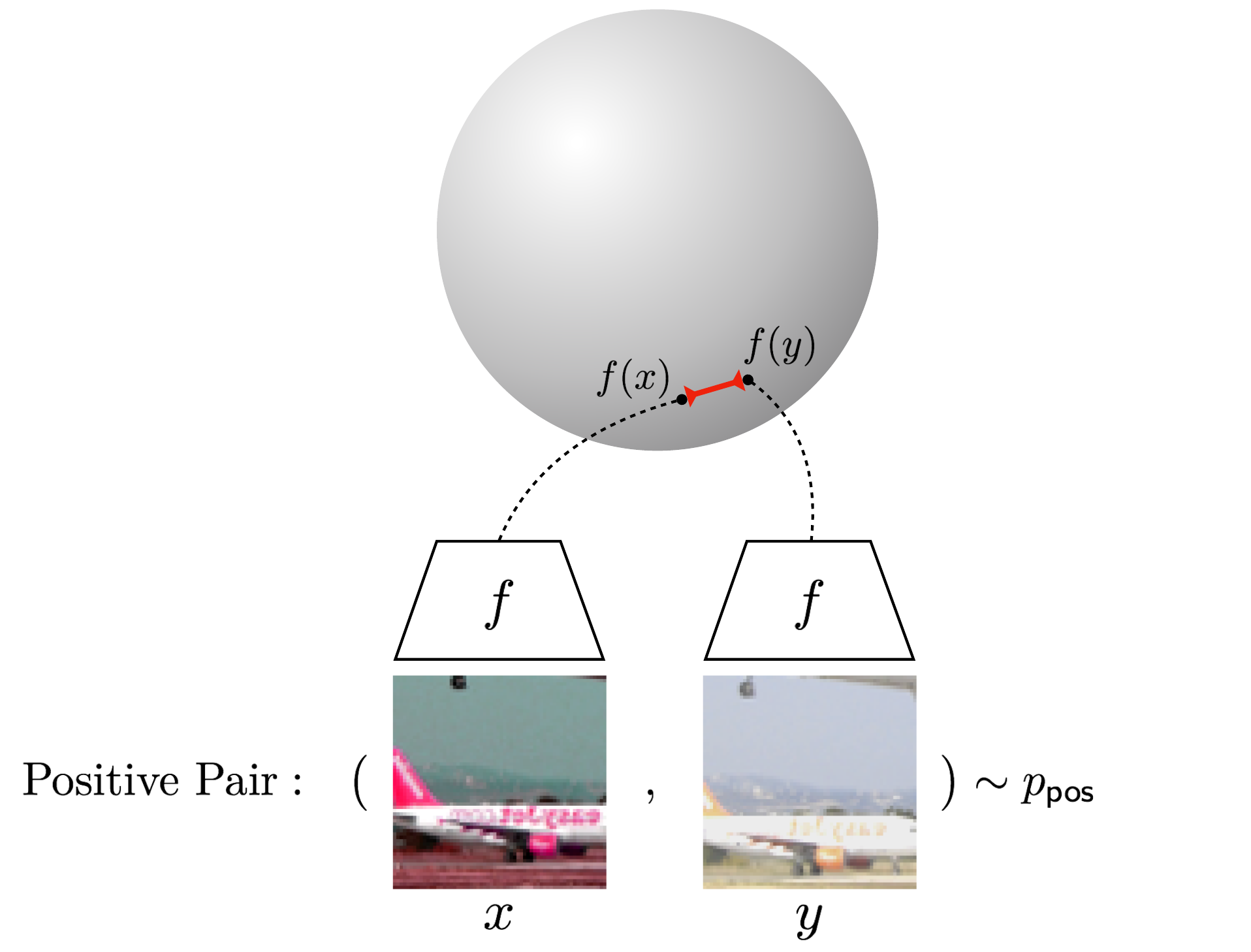}\vspace{-5pt}
    \caption*{\textbf{Alignment: }Similar samples have similar features.\\(Figure inspired by \citet{tian2019contrastive}.)}\vspace{0.5pt}
    \end{subfigure}
    \begin{subfigure}[b]{\linewidth}
    \centering
    \captionsetup{justification=centering}
    \includegraphics[page=3, width=0.88\linewidth, trim=500 225 500 245, clip]{figures/intuition.pdf}\vspace{-5pt}
    \caption*{\textbf{Uniformity: }Preserve maximal information.}
    \end{subfigure}\vspace{-7pt}
    \caption{Illustration of alignment and uniformity of feature distributions on the output unit hypersphere. \stl \citep{coates2011stl10} images are used for demonstration.} \label{fig:align_unif}\vspace{-6.5pt}
\end{figure}

\begin{figure}
    \centering\vspace{-2pt}
    \includegraphics[page=1, width=0.9\linewidth, trim=450 160 450 85, clip]{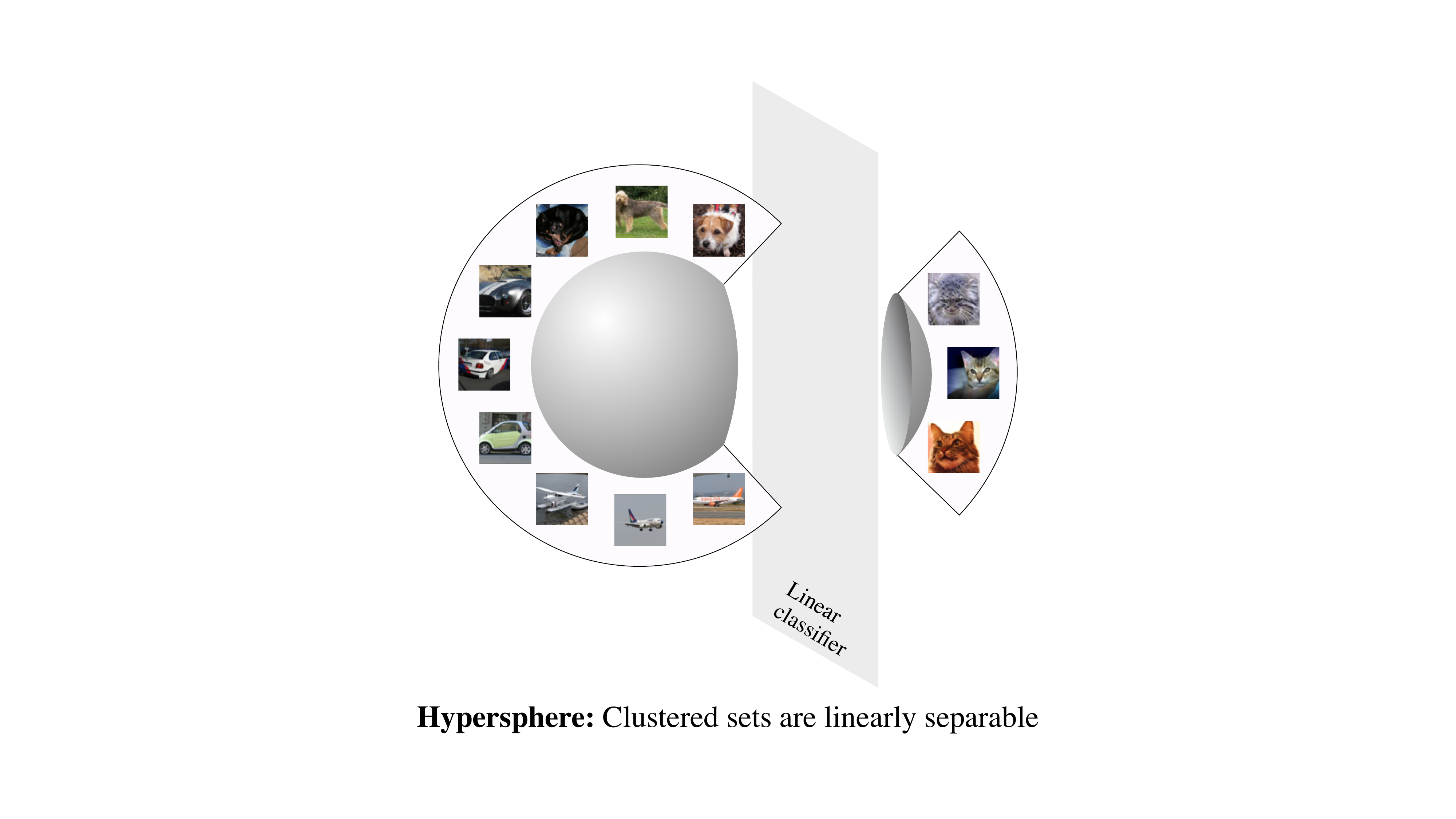}\vspace{-10pt}
    \caption{\textbf{Hypersphere: }When classes are well-clustered\hide{ on the hypersphere} (forming spherical caps), they are linearly separable. The same does not hold for Euclidean spaces.
    } \label{fig:hypersphere_linsep}\vspace{-5pt}
\end{figure}

While the unit hypersphere is a popular choice of feature space, not all encoders that map onto it are created equal. Recent works argue that representations should additionally be invariant to unnecessary details, and preserve as much information as possible \citep{oord2018representation,tian2019contrastive,hjelm2018learning,bachman2019learning}. Let us call these two properties \emph{alignment} and \emph{uniformity} (see Figure~\ref{fig:align_unif}). \emph{Alignment} favors encoders that assign similar features to similar samples. \emph{Uniformity} prefers a feature distribution that preserves maximal information, \ie, the uniform distribution on the unit hypersphere.

In this work, we analyze the \emph{alignment} and \emph{uniformity} properties. We show that a currently popular form of contrastive representation learning in fact directly optimizes for these two properties in the limit of infinite negative samples. We propose theoretically-motivated metrics for alignment and uniformity, and observe strong agreement between them and downstream task performance. Remarkably, directly optimizing for these two metrics leads to comparable or better performance than contrastive learning.

Our main contributions are:
\begin{itemize}
    \item We propose quantifiable metrics for \emph{alignment} and \emph{uniformity} as two measures of representation quality, with theoretical motivations.
    \item We prove that the contrastive loss optimizes for alignment and uniformity asymptotically.
    \item Empirically, we find strong agreement between \emph{both} metrics and downstream task performance.
    \item Despite being simple in form, our proposed metrics, when directly optimized with no other loss, empirically lead to comparable or better performance at downstream tasks than contrastive learning.
\end{itemize}


\section{Related Work}


\paragraph{Unsupervised Contrastive Representation Learning} has seen remarkable success in learning representations for image and sequential data \citep{logeswaran2018efficient,wu2018unsupervised,oord2018representation,henaff2019data,tian2019contrastive,hjelm2018learning,bachman2019learning,tian2019contrastive,he2019momentum,chen2020simple}. The common motivation behind these work is the InfoMax principle \citep{linsker1988self}, which we here instantiate as maximizing the mutual information (MI) between two views \citep{tian2019contrastive,bachman2019learning,wu2020importance}. However, this interpretation is known to be inconsistent with the actual behavior in practice, \eg, optimizing a tighter bound on MI can lead to worse representations \citep{tschannen2019mutual}. What the contrastive loss exactly does remains largely a mystery. Analysis based on the assumption of latent classes provides nice theoretical insights \citep{saunshi2019theoretical}, but unfortunately has a rather large gap with empirical practices: the result that representation quality suffers with a large number of negatives is inconsistent with empirical observations \citep{wu2018unsupervised,tian2019contrastive,he2019momentum,chen2020simple}. In this paper, we analyze and characterize the behavior of contrastive learning from the perspective of alignment and uniformity properties, and empirically verify our claims with standard representation learning tasks.

\paragraph{Representation learning on the unit hypersphere. } Outside contrastive learning, many other representation learning approaches also normalize their features to be on the unit hypersphere. In variational autoencoders, the hyperspherical latent space has been shown to perform better than the Euclidean space \citep{xu2018spherical,s-vae18}. Directly matching uniformly sampled points on the unit hypersphere is known to provide good representations \citep{bojanowski2017unsupervised}, agreeing with our intuition that uniformity is a desirable property. \citet{mettes2019hyperspherical} optimizes prototype representations on the unit hypersphere for classification. Hyperspherical face embeddings greatly outperform the unnormalized counterparts \citep{parkhi2015deep,liu2017sphereface,wang2017normface,schroff2015facenet}. Its empirical success suggests that the unit hypersphere is indeed a nice feature space. In this work, we formally investigate the interplay between the hypersphere geometry and the popular contrastive representation learning.

\paragraph{Distributing points on the unit hypersphere.}
The problem of uniformly distributing points on the unit hypersphere is a well-studied one. It is often defined as minimizing the total pairwise potential \wrt a certain kernel function \citep{borodachov2019discrete,landkof1972foundations}, \eg, the Thomson problem of finding the minimal electrostatic potential energy configuration of electrons \citep{thomson1904xxiv}, and minimization of the Riesz $s$-potential \citep{gotz2001note,hardin2005minimal,liu2018learning}. The uniformity metric we propose is based on the Gaussian potential, which can be used to represent a very general class of kernels and is closely related to the universally optimal point configurations \citep{borodachov2019discrete,cohn2007universally}. Additionally, the best-packing problem on hyperspheres (often called the Tammes problem) is also well studied \citep{tammes1930origin}.

\section{Preliminaries on Unsupervised Contrastive Representation Learning}

The popular unsupervised contrastive representation learning method (often referred to as \emph{contrastive learning} in this paper) learns representations from unlabeled data. It assumes a way to sample \emph{positive pairs}, representing similar samples that should have similar representations. Empirically, the positive pairs are often obtained by taking two independently randomly augmented versions of the same sample, \eg two crops of the same image \citep{wu2018unsupervised,hjelm2018learning,bachman2019learning,he2019momentum,chen2020simple}.

Let $\distndata(\cdot)$ be the data distribution over $\R^n$ and $\distnpos(\cdot, \cdot)$ the distribution of positive pairs over $\R^n \times \R^n$. Based on empirical practices, we assume the following property.

\begin{assumption}
    Distributions $\distndata$ and $\distnpos$ should satisfy%
    \begin{itemize}
        \item Symmetry: $\forall x, y,~ \distnpos(x, y) = \distnpos(y, x)$.
        \item Matching marginal: $\forall x,~ \int \distnpos(x, y) \diff y = \distndata(x)$.
    \end{itemize}
\end{assumption}

We consider the following specific and widely popular form of contrastive loss for training an encoder $f \colon \R^n \rightarrow \sphere^{m-1}$, mapping data to $\ell_2$ normalized feature vectors of dimension $m$. This loss has been shown effective by many recent representation learning methods \citep{logeswaran2018efficient,wu2018unsupervised,tian2019contrastive,he2019momentum,hjelm2018learning,bachman2019learning,chen2020simple}.
\begin{equation}
\begin{split}
    & \lcontr(f; \tau, M) \trieq \\
    & \quad \expectunder[\substack{
        (x, y) \sim \distnpos \\
        \{x^-_i\}_{i=1}^M \iidsim \distndata
    }]{- \log \frac{e^{f(x)\T f(y) / \tau}}{e^{f(x)\T f(y) / \tau} + \sum_i e^{f(x^-_i)\T f(y) / \tau}}},
\end{split} \label{eq:contrastive_loss}
\end{equation}
where $\tau > 0$ is a scalar temperature hyperparameter, and $M \in \Z_+$ is a fixed number of negative samples.

The term \emph{contrastive loss} has also been generally used to refer to various objectives based on positive and negative samples, \eg, in Siamese networks \citep{chopra2005learning,hadsell2006dimensionality}. In this work, we focus on the specific form in Equation~\eqref{eq:contrastive_loss} that is widely used in modern unsupervised contrastive representation learning literature.

\paragraph{Necessity of normalization.} Without the norm constraint, the $\mathtt{softmax}$ distribution can be made arbitrarily sharp by simply scaling all the features. \citet{wang2017normface} provided an analysis on this effect and argued for the necessity of normalization when using feature vector dot products in a cross entropy loss, as is in Eqn.~\eqref{eq:contrastive_loss}. Experimentally, \citet{chen2020simple} also showed that normalizing outputs leads to superior representations.

\paragraph{The InfoMax principle.} Many empirical works are motivated by the InfoMax principle of maximizing $I(f(x); f(y))$ for $(x, y) \sim \ppos$ \citep{tian2019contrastive,bachman2019learning,wu2020importance}. Usually they interpret $\lcontr$ in Eqn.~\eqref{eq:contrastive_loss} as a lower bound of $I(f(x); f(y))$ \citep{oord2018representation,hjelm2018learning,bachman2019learning,tian2019contrastive}. However, this interpretation is known to have issues in practice, \eg, maximizing a tighter bound often leads to worse downstream task performance \citep{tschannen2019mutual}. Therefore, instead of viewing it as a bound, we investigate the exact behavior of directly optimizing $\lcontr$ in the following sections.

\section{Feature Distribution on the Hypersphere}


\begin{figure*}[t!]
    \centering
    \begin{subfigure}[b]{\linewidth}
        \centering
        \includegraphics[width=\linewidth, trim=5 5 10 0, clip]{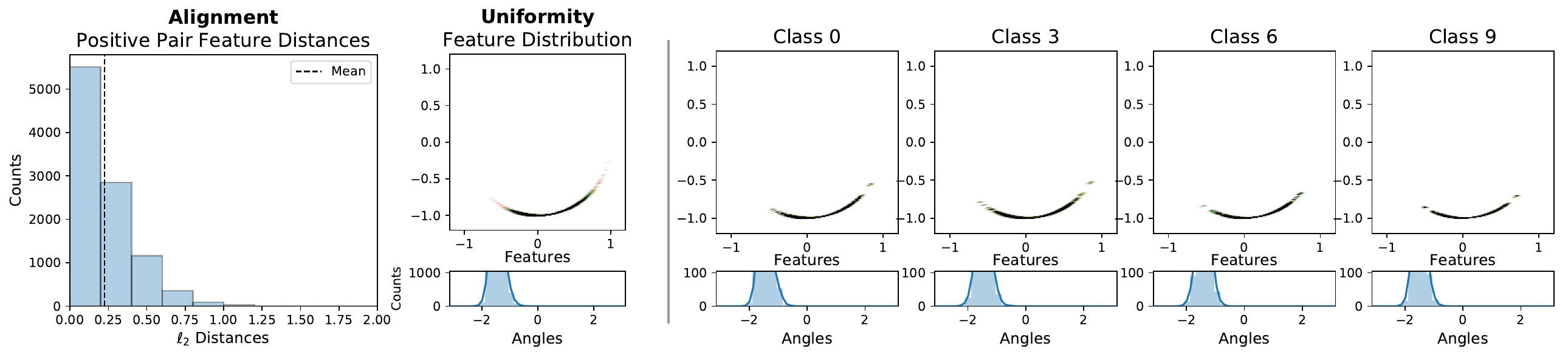}%
        \vspace{-4pt}
        \caption{\textbf{Random Initialization.} Linear classification validation accuracy: $12.71\%$.}
    \end{subfigure}
    \begin{subfigure}[b]{\linewidth}
        \centering
        \vspace{1pt}
        \includegraphics[width=\linewidth, trim=5 5 10 0, clip]{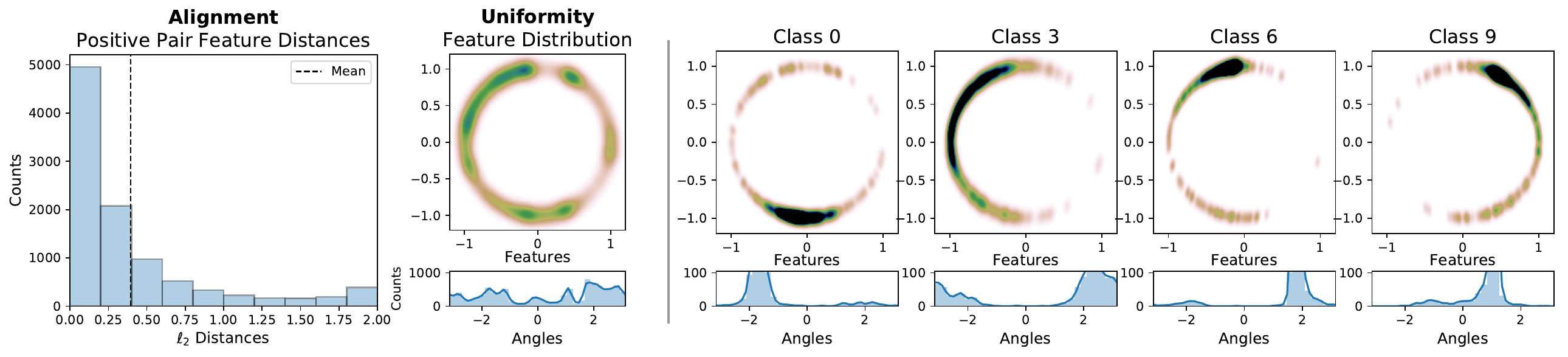}%
        \vspace{-4pt}
        \caption{\textbf{Supervised Predictive Learning.} Linear classification validation accuracy: $57.19\%$. }
    \end{subfigure}
    \begin{subfigure}[b]{\linewidth}
        \centering
        \vspace{1pt}
        \includegraphics[width=\linewidth, trim=5 5 10 0, clip]{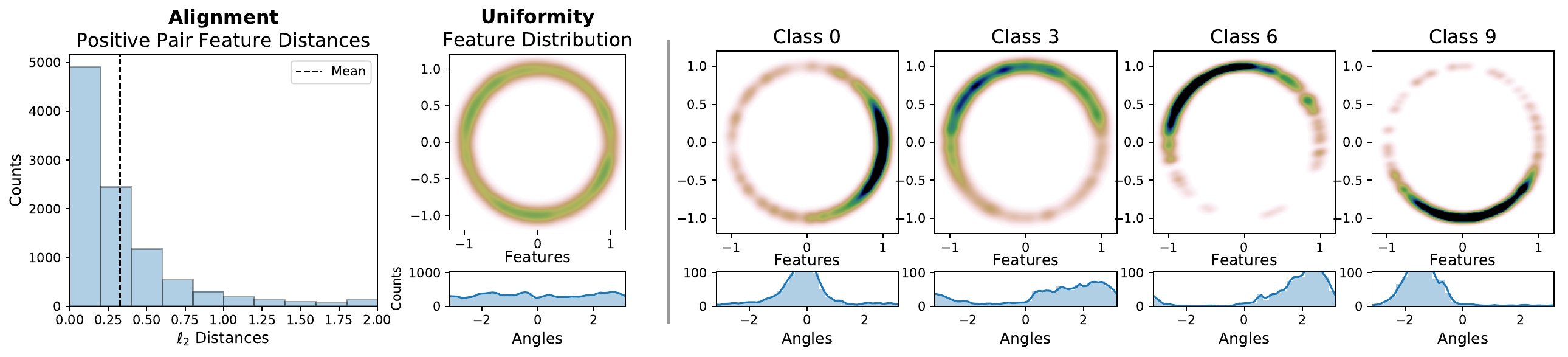}%
        \vspace{-4pt}
        \caption{\textbf{Unsupervised Contrastive Learning.} Linear classification validation accuracy: $28.60\%$.}
    \end{subfigure}\vspace{-4pt}
    \caption{Representations of \cifar validation set on $\sphere^{1}$.
    \textbf{Alignment analysis: }We show distribution of distance between features of positive pairs (two random augmentations).  \textbf{Uniformity analysis: }We plot feature distributions with Gaussian kernel density estimation (KDE) in $\R^2$ and von~Mises-Fisher (vMF) KDE on angles (\ie, $\arctantwo(y, x)$ for each point $(x, y) \in \sphere^1$). \textbf{Four rightmost plots} visualize feature distributions of selected specific classes. Representation from contrastive learning is both \emph{aligned} (having low positive pair feature distances) and \emph{uniform} (evenly distributed on $\sphere^1$).
    } \label{fig:toy2d}\vspace{-2pt}
\end{figure*}

The contrastive loss encourages learned feature representation for positive pairs to be similar, while pushing features from the randomly sampled negative pairs apart. Conventional wisdom says that representations should extract the most shared information between positive pairs and remain invariant to other noise factors \citep{linsker1988self,tian2019contrastive,wu2020importance,bachman2019learning}. Therefore, the loss should prefer two following properties: \begin{itemize}
    \item \emph{Alignment}: two samples forming a positive pair should be mapped to nearby features, and thus be (mostly) invariant to unneeded noise factors.
    \item \emph{Uniformity}: feature vectors should be roughly uniformly distributed on the unit hypersphere $\sphere^{m-1}$, preserving as much information of the data as possible.
\end{itemize}

To empirically verify this, we visualize \cifar \citep{torralba200880,krizhevsky2009learning} representations on $\sphere^1$ ($m = 2$) obtained via three different methods: \begin{itemize}
    \item Random initialization.
    \item Supervised predictive learning: An encoder and a linear classifier are jointly trained from scratch with cross entropy loss on supervised labels.
    \item Unsupervised contrastive learning: An encoder is trained \wrt $\lcontr$ with $\tau = 0.5$ and $M = 256$.
\end{itemize} All three encoders share the same AlexNet based architecture \citep{krizhevsky2012imagenet}, modified to map input images to $2$-dimensional vectors in $\sphere^1$. Both predictive and contrastive learning use standard data augmentations to augment the dataset and sample positive pairs.

Figure~\ref{fig:toy2d} summarizes the resulting distributions of validation set features. Indeed, features from unsupervised contrastive learning (bottom in Figure~\ref{fig:toy2d}) exhibit the most uniform distribution, and are closely clustered for positive pairs.

The form of the contrastive loss in Eqn.~\eqref{eq:contrastive_loss} also suggests this. We present informal arguments below, followed by more formal treatment in Section~\ref{sec:limit}. From the symmetry of $p$, we can derive%
\begin{equation*}
\begin{split}
    & \lcontr(f; \tau, M) = \expectunder[(x, y) \sim \distnpos]{- f(x)\T f(y) / \tau} \\
    & \quad +\hspace{-5pt}\expectundernear[\substack{
    (x, y) \sim \distnpos \\
    \{x^-_i\}_{i=1}^{M} \iidsim \distndata
    }]{\log\hspace{-2pt}\left( e^{f(x)\T f(y) / \tau}+ \sum_i e^{f(x^-_i)\T f(x) / \tau} \right)}.
\end{split}
\end{equation*}
Because the $\sum_i e^{f(x^-_i)\T f(x) / \tau}$ term is always positive and bounded below, the loss favors smaller $\expect{- f(x)\T f(y) / \tau}$, \ie, having more aligned positive pair features. Suppose the encoder is perfectly aligned, \ie, $\prob{f(x) = f(y)} = 1$, then minimizing the loss is equivalent to optimizing \begin{equation*}
    \expectunder[\substack{
    x \sim \distndata \\
    \{x^-_i\}_{i=1}^{M} \iidsim \distndata
    }]{\log \left( e^{1 / \tau}+ \sum_i e^{f(x^-_i)\T f(x) / \tau} \right)},
\end{equation*}
which is akin to maximizing pairwise distances with a $\mathtt{LogSumExp}$ transformation. Intuitively, pushing all features away from each other should indeed cause them to be roughly uniformly distributed.


\subsection{Quantifying Alignment and Uniformity} \label{sec:unif-align}

For further analysis, we need a way to measure alignment and uniformity. We propose the following two metrics (losses).


\begin{figure*}[ht!]\vspace{-1pt}
    \centering
    \hspace{-0.005\linewidth}\includegraphics[width=1.01\linewidth]{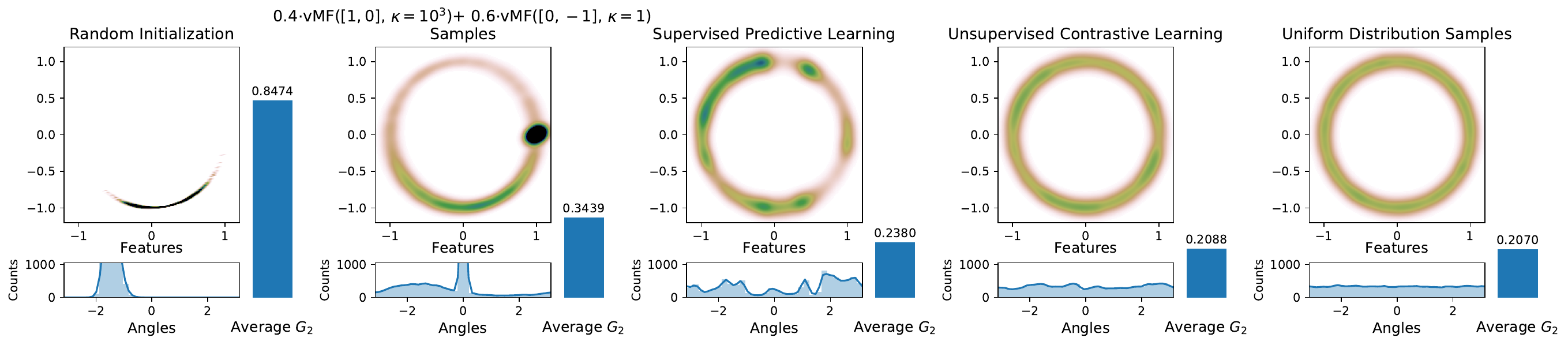}%
    \vspace{-9pt}
    \caption{Average pairwise $G_{2}$ potential as a measure of uniformity. Each plot shows $10000$ points distributed on $\sphere^1$, obtained via either applying an encoder on \cifar validation set (same as those in Figure~\ref{fig:toy2d}) or sampling from a distribution on $\sphere^1$, as described in plot titles. We show the points with Gaussian KDE and the angles with vMF KDE.} \label{fig:gaussian_vis}\vspace{-6pt}
\end{figure*}

\subsubsection{Alignment}
The alignment loss is straightforwardly defined with the expected distance between positive pairs:  %
\begin{equation*}
    \lalign(f; \alpha) \trieq \expectunder[(x, y) \sim \distnpos] {\norm{f(x) - f(y)}_2^\alpha}, \quad \alpha > 0.
\end{equation*}

\subsubsection{Uniformity} \label{sec:unif}
We want the uniformity metric to be both asymptotically correct (\ie, the distribution optimizing this metric should converge to uniform distribution) and empirically reasonable with finite number of points. To this end, we consider the Gaussian potential kernel (also known as the Radial Basis Function (RBF) kernel) $G_t \colon \sphere^{d} \times \sphere^{d} \rightarrow \R_+$ \citep{cohn2007universally,borodachov2019discrete}: \begin{equation*}
    G_t(u, v) \trieq e^{-t \norm{u - v}_2^2} = e^{2t \cdot u\T v - 2t}, \quad t > 0,
\end{equation*}
and define the uniformity loss as the logarithm of the average pairwise Gaussian potential: \begin{align*}
    \lunif(f; t)
    & \trieq \log \expectunder[\hspace{2pt}x, y \iidsim \distndata] {G_t(u, v)} \\
    & = \log \expectunder[\hspace{2pt}x, y \iidsim \distndata] {e^{-t \norm{f(x) - f(y)}_2^2}}
    , \quad t > 0.%
\end{align*}%

The average pairwise Gaussian potential is nicely tied with the uniform distribution on the unit hypersphere.

\begin{definition}[Uniform distribution on $\sphere^{d}$]
    $\sigma_{d}$ denotes the normalized surface area measure on $\sphere^{d}$.
\end{definition}

First, we show that the uniform distribution is the unique distribution that minimize the expected pairwise potential.
\begin{proposition} \label{prop:continuous-problem}
    For $\mathcal{M}(\sphere^{d})$ the set of Borel probability measures on $\sphere^{d}$, $\sigma_{d}$ is the unique solution of \begin{equation*}
        \min_{\mu \in \mathcal{M}(\sphere^{d})} \int_u \int_v G_t(u, v) \diff \mu \diff \mu.
    \end{equation*}
\end{proposition}
\begin{proof}
    See \suppmat.
\end{proof}

In addition, as number of points goes to infinity, distributions of points minimizing the average pairwise potential converge weak$^*$ to the uniform distribution. Recall the definition of the weak$^*$ convergence of measures.

\begin{definition}[Weak$^*$ convergence of measures]
    A sequence of Borel measures $\{\mu_n\}_{n=1}^\infty$ in $\R^p$ converges weak$^*$ to a Borel measure $\mu$ if for all continuous function $f \colon \R^p \rightarrow \R$, we have \begin{equation*}
        \lim_{n\rightarrow \infty} \int f(x) \diff \mu_n(x) = \int f(x) \diff \mu(x).
    \end{equation*}
\end{definition}

\begin{proposition} \label{prop:discrete-problems}
    For each $N > 0$, the $N$ point minimizer of the average pairwise potential is \begin{equation*}
        \mathbf{u}^*_N = \argmin_{u_1, u_2, \dots, u_N \in \sphere^{d}} \sum_{1 \leq i < j \leq N} G_t(u_i, u_j).
    \end{equation*}
    The normalized counting measures associated with the $\{\mathbf{u}^*_N\}_{N=1}^\infty$ sequence converge weak$^*$ to $\sigma_d$.
\end{proposition}
\begin{proof}
    See \suppmat.
\end{proof}

Designing an objective minimized by the uniform distribution is in fact nontrivial. For instance, average pairwise dot products or Euclidean distances is simply optimized by any distribution that has zero mean. Among kernels that achieve uniformity at optima, the Gaussian kernel is special in that it is closely related to the universally optimal point configurations and can also be used to represent a general class of other kernels, including the Riesz $s$-potentials. We refer readers to \citet{borodachov2019discrete}~and~\citet{cohn2007universally} for in-depth discussions on these topics. Moreover, as we show below, $\lunif$, defined with the Gaussian kernel, has close connections with $\lcontr$.

Empirically, we evaluate the average pairwise potential of various finite point collections on $\sphere^1$ in Figure~\ref{fig:gaussian_vis}. The values nicely align with our intuitive understanding of uniformity.

We further discuss  properties of $\lunif$ and characterize its optimal value and range in the \suppmat.

\subsection{Limiting Behavior of Contrastive Learning} \label{sec:limit}

In this section, we formalize the intuition that contrastive learning optimizes alignment and uniformity, and characterize its asymptotic behavior. We consider optimization problems over all measurable encoder functions from the $\pdata$ measure in $\R^n$ to the Borel space $\sphere^{m-1}$.

We first define the notion of optimal encoders for each of these two metrics.

\begin{definition}[Perfect Alignment]
    We say an encoder $f$ is \emph{perfectly aligned} if $f(x) = f(y)$ \asurely over $(x, y) \sim \distnpos$.
\end{definition}

\begin{definition}[Perfect Uniformity]
    We say an encoder $f$ is \emph{perfectly uniform} if %
    the distribution of $f(x)$ for $x \sim \distndata$ is the uniform distribution $\sigma_{m-1}$ on $\sphere^{m-1}$.
\end{definition}

\paragraph{Realizability of perfect uniformity.} We note that it is not always possible to achieve perfect uniformity, \eg, when the data manifold in $\R^n$ is lower dimensional than the feature space $\sphere^{m-1}$. Moreover, in the case that $\distndata$ and $\distnpos$ are formed from sampling augmented samples from a finite dataset, there cannot be an encoder that is \emph{both} perfectly aligned and perfectly uniform, because perfect alignment implies that all augmentations from a single element have the same feature vector. Nonetheless, perfectly uniform encoder functions do exist under the conditions that $n \geq m-1$ and $\distndata$ has bounded density. 

We analyze the asymptotics with infinite negative samples. Existing empirical work has established that larger number of negative samples consistently leads to better downstream task performances \citep{wu2018unsupervised,tian2019contrastive,he2019momentum,chen2020simple}, and often uses very large values (\eg, $M=65536$ in \citet{he2019momentum}). The following theorem nicely confirms that optimizing \wrt the limiting loss indeed requires both alignment and uniformity.

\begin{theorem}[Asymptotics of $\lcontr$] \label{thm:asym_inf_negatives}
    For fixed $\tau > 0$, as the number of negative samples $M \rightarrow \infty$, the (normalized) contrastive loss converges to \begin{equation}
        \begin{split}
        & \lim_{M \rightarrow \infty} \lcontr(f; \tau, M) - \log M = \\
        & \qquad -\frac{1}{\tau} \expectunder[(x, y) \sim \distnpos]{f(x)\T f(y)} \\
        & \qquad + \expectunder[x \sim \distndata]{\log\expectunder[x^- \sim \distndata]{e^{f(x^-)\T f(x) / \tau}}}.
        \end{split}
        \label{eq:contrastive_loss_limit}
    \end{equation}
    We have the following results: \begin{enumerate}
        \item The first term is minimized iff $f$ is perfectly aligned.
        \item If perfectly uniform encoders exist, they form the exact minimizers of the second term.
        \item For the convergence in Equation~\eqref{eq:contrastive_loss_limit}, the absolute deviation from the limit decays in $\mathcal{O}(M^{-1/2})$.
    \end{enumerate}
\end{theorem}
\begin{proof}
    See \suppmat.
\end{proof}

\paragraph{Relation with $\lunif$. }The proof of Theorem~\ref{thm:asym_inf_negatives} in the \suppmat connects the asymptotic $\lcontr$ form with minimizing average pairwise Gaussian potential, \ie, minimizing $\lunif$. Compared with the second term of Equation~\eqref{eq:contrastive_loss_limit}, $\lunif$ essentially pushes the $\log$ outside the outer expectation, without changing the minimizer (perfectly uniform encoders). However, due to its pairwise nature, $\lunif$ is much simpler in form and avoids the computationally expensive $\mathtt{softmax}$ operation in $\lcontr$ \citep{goodman2001classes,bengio2003quick,gutmann2010noise,grave2017efficient,chen2018learning}.


\paragraph{Relation with feature distribution entropy estimation.} When $\pdata$ is uniform over finite samples $\{x_1, x_2, \dots, x_N\}$ (\eg, a collected dataset), the second term in Equation~\eqref{eq:contrastive_loss_limit} can be alternatively viewed as a resubstitution entropy estimator of $f(x)$ \citep{ahmad1976nonparametric}, where $x$ follows the underlying distribution $p_\mathsf{nature}$ that generates $\{x_i\}_{i=1}^N$, via a von~Mises-Fisher (vMF) kernel density estimation (KDE):%
\begingroup%
\setlength{\belowdisplayskip}{2pt}%
\begin{align*}
    & \expectunder[x \sim \distndata]{\log\expectunder[x^- \sim \pdata]{e^{f(x^-)\T f(x) / \tau}}} \\
    & \qquad = \frac{1}{N}\sum_{i=1}^N \log \left( \frac{1}{N}\sum_{j=1}^N e^{f(x_i)\T f(x_j) / \tau} \right) \\
    & \qquad = \frac{1}{N}\sum_{i=1}^N \log \hat{p}_\mathsf{vMF\text{-}KDE}(f(x_i)) + \log Z_\mathsf{vMF} \\
    & \qquad \trieq -\hat{H}(f(x)) + \log Z_\mathsf{vMF}, \qquad\qquad\quad x \sim p_\mathsf{nature} \\
    & \qquad \trieq -\hat{I}(x; f(x)) + \log Z_\mathsf{vMF}, \qquad\quad\hspace{14pt}x \sim p_\mathsf{nature},
\end{align*}%
\endgroup%
where \begin{itemize}
    \item $\hat{p}_\mathsf{vMF\text{-}KDE}$ is the KDE based on samples $\{f(x_j)\}_{j=1}^N$ using a vMF kernel with $\kappa = \tau^{-1}$,
    \item $Z_\mathsf{vMF}$ is the normalization constant for vMF distribution with $\kappa = \tau^{-1}$,
    \item $\hat{H}$ denotes the resubstitution entropy estimator,
    \item $\hat{I}$ denotes the mutual information estimator based on $\hat{H}$, since $f$ is a deterministic function.
\end{itemize}

\paragraph{Relation with the InfoMax principle.}  Many empirical works are motivated by the InfoMax principle, \ie, maximizing $I(f(x); f(y))$ for $(x, y) \sim \ppos$.  However, the interpretation of $\lcontr$  as a lower bound of $I(f(x); f(y))$ is known to be inconsistent with its actual behavior in practice \citep{tschannen2019mutual}. Our results instead analyze the properties of $\lcontr$ itself. Considering the identity $I(f(x); f(y)) = H(f(x)) - H(f(x) \given f(y))$, we can see that while uniformity indeed favors large $H(f(x))$, alignment is stronger than merely desiring small $H(f(x) \given f(y))$. In particular, both Theorem~\ref{thm:asym_inf_negatives} and the above connection with maximizing an entropy estimator provide alternative interpretations and motivations that $\lcontr$ optimizes for \emph{aligned} and \emph{information-preserving} encoders.

Finally, even for the case where only a single negative sample is used (\ie, $M=1$), we can still prove a weaker result, which we describe in details in the \suppmat.




\begin{figure*}[t]
    \centering
    \begin{subfigure}[t]{0.6505\linewidth}
        \includegraphics[width=0.4935\linewidth, trim=5 5 5 5, clip]{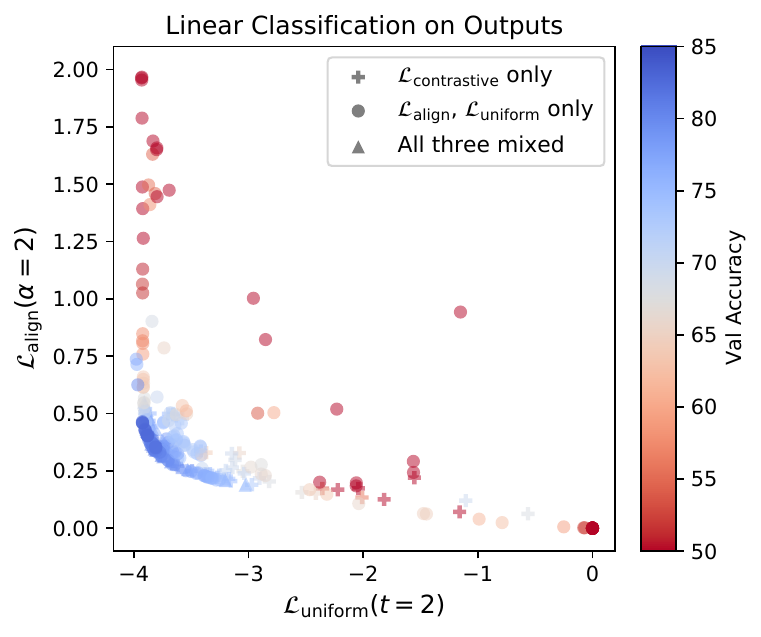}
        \hfill
        \includegraphics[width=0.4935\linewidth, trim=5 5 5 5, clip]{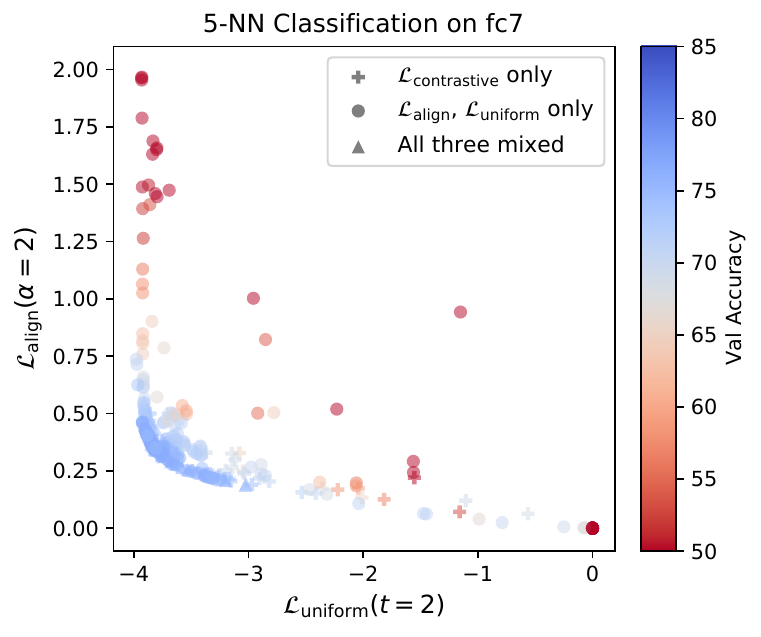}%
        \vspace{-4pt}
        \caption{$304$ \stl encoders are evaluated with linear classification on output features and $5$-nearest neighbor ($5$-NN) on \textrm{fc7} activations. Higher accuracy (blue color) is better.}\label{fig:expr_scatter_stl10}
    \end{subfigure}%
    \hfill%
    \begin{subfigure}[t]{0.33\linewidth}
        \includegraphics[width=0.983\linewidth, trim=5 4.5 5 5, clip]{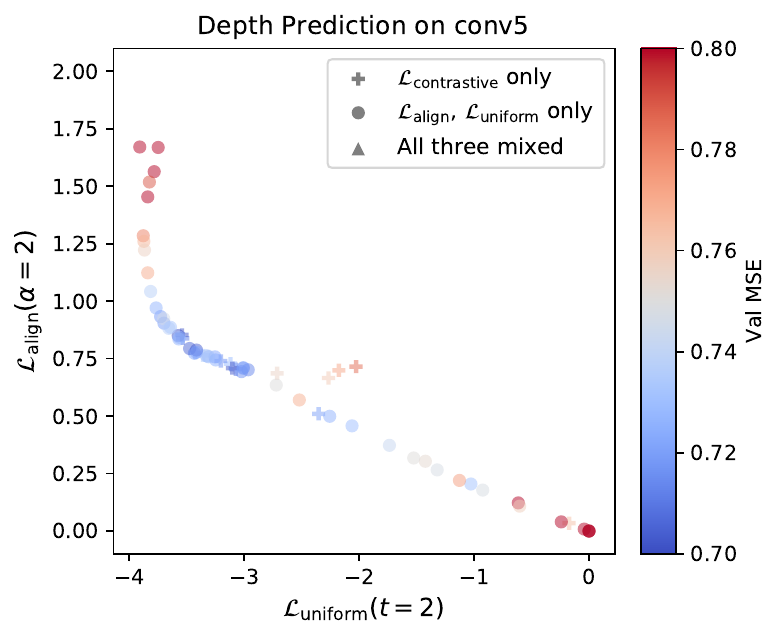}%
        \vspace{-4pt}
        \caption{$64$ \nyudepth encoders are evaluated with CNN depth regressors on \textrm{conv5} activations. Lower MSE (blue color) is better.}
    \end{subfigure}\vspace{-6pt}
    \caption{Metrics and performance of \stl and \nyudepth experiments. Each point represents a trained encoder, with its $x$- and $y$-coordinates showing $\lalign$ and $\lunif$ metrics and color showing the performance on validation set. \textbf{Blue} is better for both tasks. Encoders with low $\lalign$ and $\lunif$ are consistently the better performing ones (lower left corners).}\label{fig:expr_scatter}\vspace{-2.8pt}
\end{figure*}

\begin{figure}[t]
    \input{figuretext/experiments/code_lalign_lunif}%
    \vspace{-15pt}
    \caption{PyTorch implementation of $\lalign$ and $\lunif$.} \label{fig:pytorch-losses-code}%
    \vspace{-1pt}
\end{figure}%

\section{Experiments}
In this section, we empirically verify the hypothesis that alignment and uniformity are desired properties for representations. Recall that our two metrics are \begin{align*}
    \lalign(f; \alpha) & \trieq \expect[(x, y) \sim \distnpos] {\norm{f(x) - f(y)}_2^\alpha} \\
    \lunif(f; t) & \trieq \log~\expect[x, y \iidsim \distndata] {e^{-t \norm{f(x) - f(y)}_2^2}}.
\end{align*}\vspace{-10pt}

We conduct extensive experiments with convolutional neural network (CNN) and recurrent neural network (RNN) based encoders on four popular representation learning benchmarks with distinct types of downstream tasks: \begin{itemize}
    \item \stl \citep{coates2011stl10} classification on AlexNet-based encoder outputs or intermediate activations with a linear or $k$-nearest neighbor ($k$-NN) classifier.
    \item \nyudepth \citep{Silberman:ECCV12} depth prediction on CNN encoder intermediate activations after convolution layers.
    \item \imagenet and \imagenetsubset (random $100$-class subset of \imagenet) classification on CNN encoder penultimate layer activations with a linear classifier.
    \item \bookcorpus \citep{moviebook} RNN sentence encoder outputs used for Moview Review Sentence Polarity (\moviereview) \citep{pang2005seeing} and Customer Product Review Sentiment (\customerreview) \citep{wang2012baselines} binary classification tasks with logisitc classifiers.
\end{itemize}

For image datasets, we follow the standard practice and choose positive pairs as two independent augmentations of the same image. For \bookcorpus, positive pairs are chosen as neighboring sentences, following Quick-Thought Vectors \citep{logeswaran2018efficient}. 

We perform majority of our analysis on \stl and \nyudepth encoders, where we calculate $\lcontr$ with negatives being other samples within the minibatch following the standard practice \citep{hjelm2018learning,bachman2019learning,tian2019contrastive,chen2020simple}, and $\lunif$ as the logarithm of average pairwise feature potentials also within the minibatch. Due to their simple forms, these two losses can be implemented in PyTorch \citep{paszke2019pytorch} with less than $10$ lines of code, as shown in Figure~\ref{fig:pytorch-losses-code}.

To investigate \emph{alignment} and \emph{uniformity} properties on recent contrastive learning methods and larger datasets, we also analyze \imagenet and \imagenetsubset encoders trained with Momentum Contrast (MoCo) \citep{he2019momentum,chen2020improved}, and \bookcorpus encoders trained with Quick-Thought Vectors \citep{logeswaran2018efficient}, with these methods modified to also allow $\lalign$ and $\lunif$.

\begin{table*}[t!]%
    \centering

\resizebox{
  0.9\width
}{!}{%
\small
\centering
\newcommand{\mroundprec}[1]{\round{#1}{2}\%}%
\renewcommand{\arraystretch}{1.4}%
\begin{tabular}{|c||c||c|c|c|c|}
    \hline
    &  \multirow{2}{*}{\vspace{-1pt}Loss Formula}
    &  \multicolumn{4}{c|}{Validation Set Accuracy $\uparrow$} \\
    \cline{3-6}
    &
    &  Output + Linear
    &  Output + $5$-NN
    &  \textrm{fc7} + Linear
    &  \textrm{fc7} + $5$-NN \\
    \hline\hline

    Best $\lcontr$ only
    & $\lcontr(\tau\narroweq0.19)$ 
    & \mroundprec{80.4625} 
    & \mroundprec{78.75} 
    & \mroundprec{83.8875} 
    & \mroundprec{76.325} 
    \\
    \hline

    Best $\lalign$ and $\lunif$ only
    & $0.98 \cdot \lalign(\alpha\narroweq2) + 0.96 \cdot \lunif(t\narroweq2)$ 
    & \textbf{\mroundprec{81.15}} 
    & \mroundprec{78.8875} 
    & \textbf{\mroundprec{84.425}} 
    & \textbf{\mroundprec{76.775}} 
    \\
    \hline

    Best among all encoders
    & $\lcontr(\tau\narroweq0.5) + \lunif(t\narroweq2)$ 
    & \mroundprec{81.0625} 
    & \textbf{\mroundprec{79.05}} 
    & \mroundprec{84.1375} 
    & \mroundprec{76.475} 
    \\



    \hline
\end{tabular}%
}%
\vspace{-5pt}%
\caption{\stl encoder evaluations. Numbers show linear and $5$-nearest neighbor ($5$-NN) classification accuracies on the validation set. The best result is picked by encoder outputs linear classifier accuracy from a $5$-fold training set cross validation, among all $150$ encoders trained from scratch with $128$-dimensional output and $768$ batch size. }
\label{tbl:expr_stl10}

\end{table*}
\begin{table*}[t!]%
    \vspace{-4pt}
    \centering

\vspace{-1pt}%
\resizebox{
  0.9\width
}{!}{%
\small
\centering
\newcommand{\mround}[1]{\round{#1}{4}}%
\renewcommand{\arraystretch}{1.4}%
\begin{tabular}{|c||c||c|c|}
    \hline
    &  \multirow{2}{*}{\vspace{-1pt}Loss Formula}
    &  \multicolumn{2}{c|}{Validation Set MSE $\downarrow$} \\
    \cline{3-4}
    &
    &  \hspace{8pt}\textrm{conv5}\hspace{8pt}
    &  \hspace{8pt}\textrm{conv4}\hspace{8pt} \\
    \hline\hline

    Best $\lcontr$ only
    & $0.5 \cdot \lcontr(\tau\narroweq0.1)$
    & \mround{0.7023918628692627}  
    & \textbf{\mround{0.7574809789657593}} 
    \\
    \hline

    Best $\lalign$ and $\lunif$ only
    & $0.75 \cdot \lalign(\alpha\narroweq2) + 0.5 \cdot \lunif(t\narroweq2)$
    & \textbf{\mround{0.7014151811599731}}
    & \mround{0.7591848969459534}
    \\
    \hline

    Best among all encoders
    & $0.75 \cdot \lalign(\alpha\narroweq2) + 0.5 \cdot \lunif(t\narroweq2)$
    & \textbf{\mround{0.7014151811599731}}
    & \mround{0.7591848969459534}
    \\



    \hline
\end{tabular}%
}%
\vspace{-5pt}%
\caption{\nyudepth encoder evaluations. Numbers show depth prediction mean squared error (MSE) on the validation set. The best result is picked based on \textrm{conv5} layer MSE from a $5$-fold training set cross validation, among all $64$ encoders trained from scratch with $128$-dimensional output and $128$ batch size.}
\label{tbl:expr_nyudepth}
    \vspace{-11.5pt}
\end{table*}


\begin{figure}[t]\vspace{-3pt}
    \centering
    \includegraphics[width=0.98\linewidth]{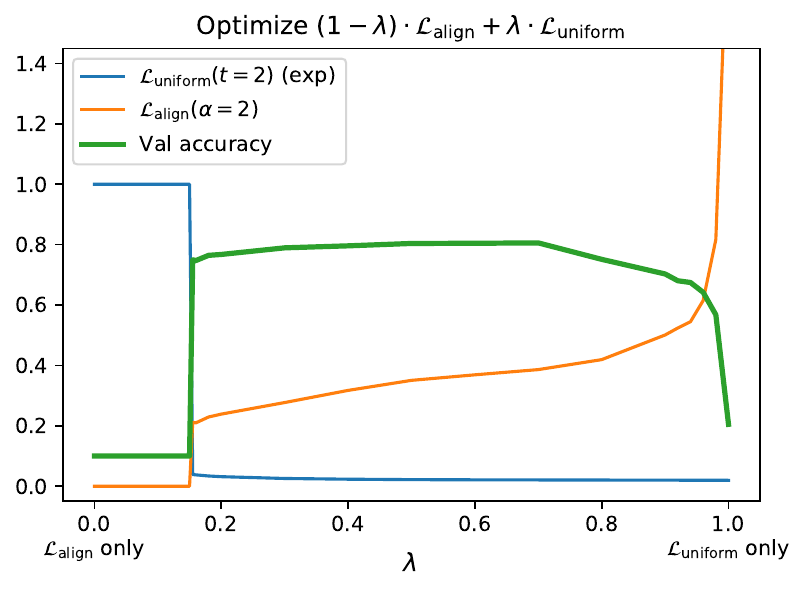}
    \vspace{-13.5pt}
    \caption{Effect of optimizing different weighted combinations of $\lalign(\alpha\narroweq2)$ and $\lunif(t\narroweq2)$ for \stl. For each encoder, we show the $\lalign$ and $\lunif$ metrics, and validation accuracy of a linear classifier trained on encoder outputs. $\lunif$ is exponentiated for plotting purposes. } \label{fig:expr_lgda_trend}\vspace{-8pt}
\end{figure}

We optimize a total of $304$ \stl encoders, $64$ \nyudepth encoders, $45$ \imagenetsubset encoders, and $108$ \bookcorpus encoders without supervision. The encoders are optimized \wrt weighted combinations of $\lcontr$, $\lalign$, and/or $\lunif$, with varying \begin{itemize}
    \item (possibly zero) weights on the three losses,
    \item temperature $\tau$ for $\lcontr$,
    \item $\alpha \in \{1, 2\}$ for $\lalign$,
    \item $t \in \{1, 2,\dots, 8\}$ for $\lunif$,
    \item batch size (affecting the number of (negative) pairs for $\lcontr$ and $\lunif$),
    \item embedding dimension,
    \item number of training epochs and learning rate,
    \item initialization (from scratch vs.~a pretrained encoder).
\end{itemize}
See the \suppmat for more experiment details and the exact configurations used.

\paragraph{$\lalign$ and $\lunif$ strongly agree with downstream task performance.}
For each encoder, we measure the downstream task performance, and the $\lalign$, $\lunif$ metrics on the validation set. Figure~\ref{fig:expr_scatter} visualizes the trends between both metrics and representation quality. We observe that the two metrics strongly agrees the representation quality overall. In particular, the best performing encoders are exactly the ones with low $\lalign$ and $\lunif$, \ie, the lower left corners in Figure~\ref{fig:expr_scatter}.

\paragraph{Directly optimizing only $\lalign$ and $\lunif$ can lead to better representations.}
As shown in Tables~\ref{tbl:expr_stl10}~and~\ref{tbl:expr_nyudepth}, encoders trained with only $\lalign$ and $\lunif$ consistently outperform their $\lcontr$-trained counterparts, for both tasks. Theoretically, Theorem~\ref{thm:asym_inf_negatives} showed that $\lcontr$ optimizes alignment and uniformity asymptotically with infinite negative samples. This empirical performance gap suggests that directly optimizing these properties can be superior in practice, when we can only have finite negatives.


\begin{figure*}[t]\vspace{2pt}
    \centering
    \hspace*{-5pt}\begin{subfigure}[b]{0.325\linewidth}
        \centering
        \includegraphics[height=0.745\linewidth]{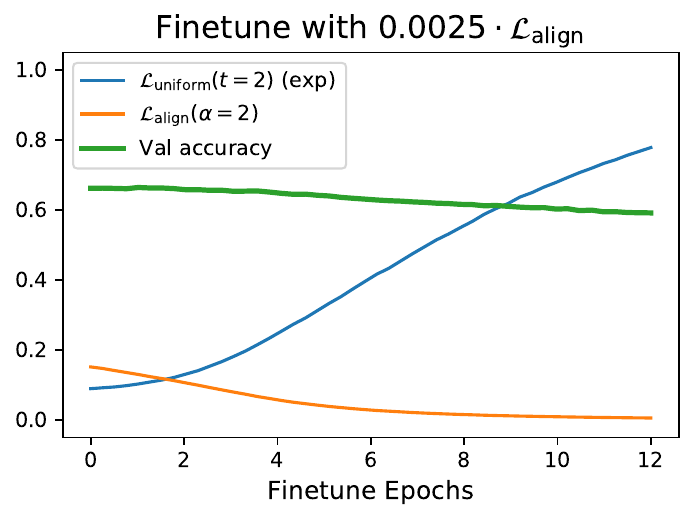}
    \end{subfigure}\hfill%
    \begin{subfigure}[b]{0.325\linewidth}
        \centering
        \includegraphics[height=0.745\linewidth]{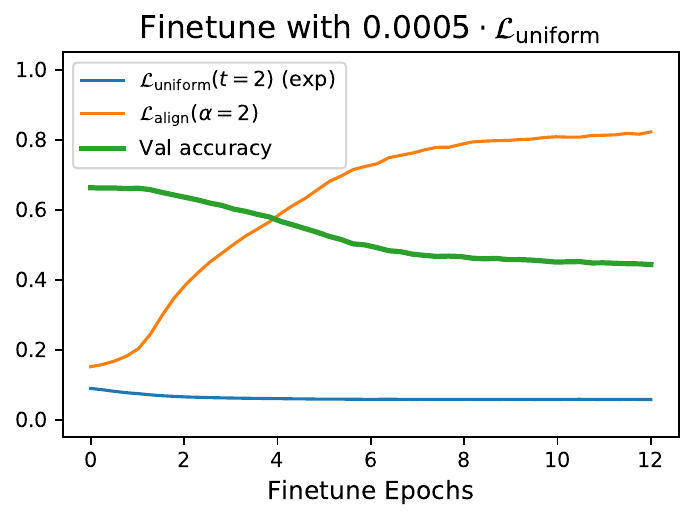}
    \end{subfigure}\hfill%
    \begin{subfigure}[b]{0.325\linewidth}
        \centering
        \includegraphics[height=0.745\linewidth]{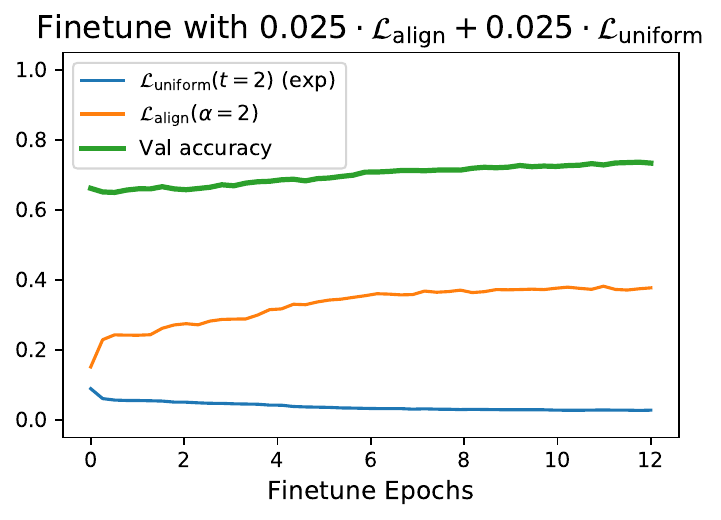}
    \end{subfigure}\vspace{-10pt}
    \caption{Finetuning trajectories from a \stl encoder trained with $\lcontr$ using a suboptimal temperature $\tau = 2.5$. Finetuning objectives are weighted combinations of $\lalign(\alpha\narroweq2)$ and $\lunif(t\narroweq2)$. For each intermediate checkpoint, we measure $\lalign$ and $\lunif$ metrics, as well as validation accuracy of a linear classifier trained from scratch on the encoder outputs. $\lunif$ is exponentiated for plotting purpose. \textbf{Left and middle: }Performance degrades if only one of alignment and uniformity is optimized. \textbf{Right: }Performance improves when both are optimized.} \label{fig:expr_finetune}%
\end{figure*}

\begin{figure*}[t]\vspace{2pt}
    \centering
    \begin{subfigure}[t]{0.32\linewidth}
        \includegraphics[width=1.005\linewidth, trim=5 4.5 5 5, clip]{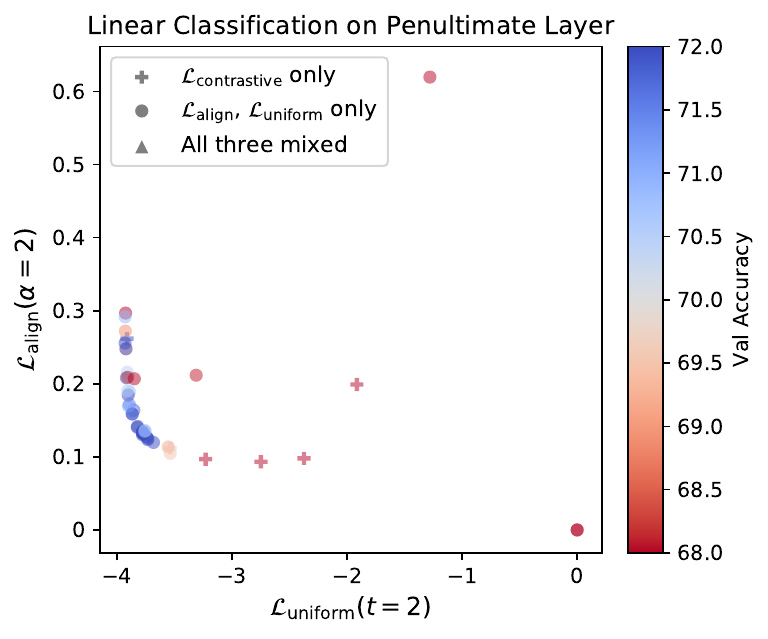}%
        \vspace{-4pt}
        \caption{$45$ \imagenetsubset encoders are trained with MoCo-based methods, and evaluated with linear classification\hide{ on penultimate layer activations}. }\label{fig:expr_scatter_imagenet100}
    \end{subfigure}%
    \hfill%
    \begin{subfigure}[t]{0.6515\linewidth}
        \includegraphics[width=0.49\linewidth, trim=5 5 5 5, clip]{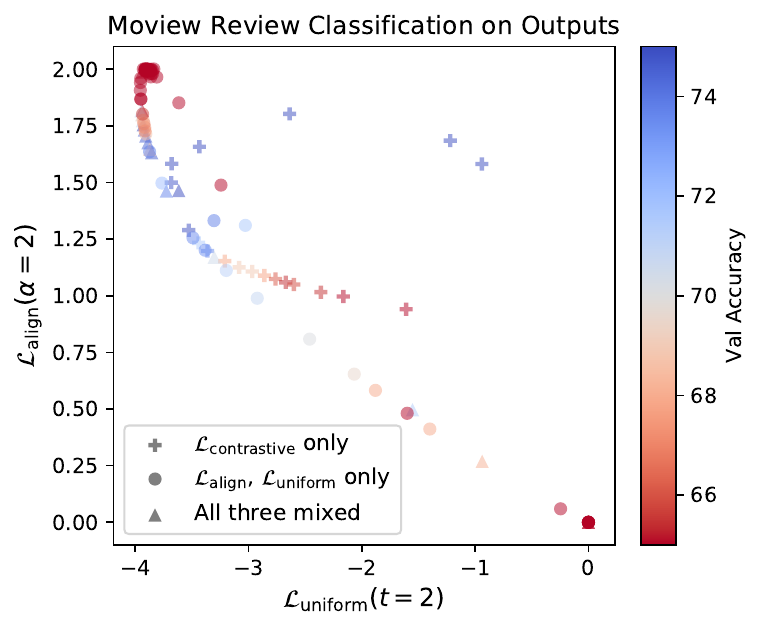}
        \hfill
        \includegraphics[width=0.49\linewidth, trim=5 5 5 5, clip]{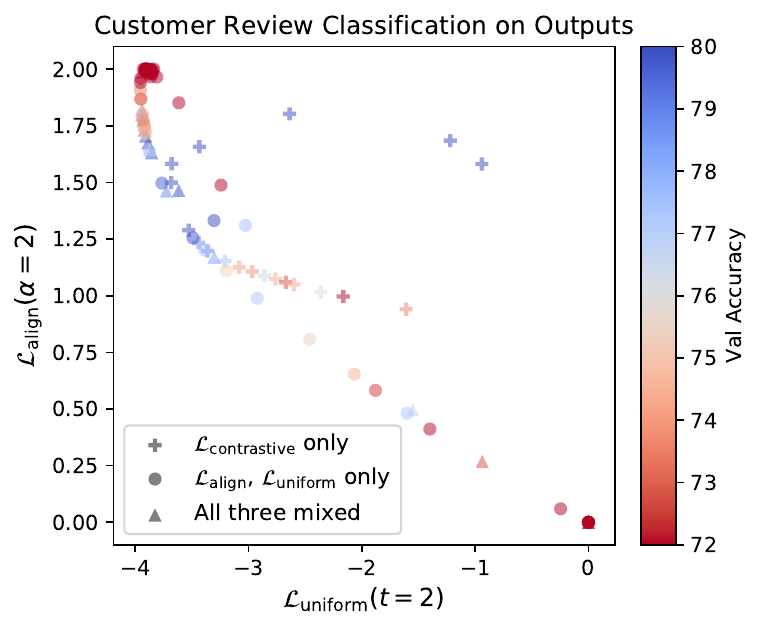}%
        \vspace{-4pt}
        \caption{$108$ \bookcorpus encoders are trained with Quick-Thought-Vectors-based methods, and evaluated with logistic binary classification on Movie Review Sentence Polarity and Customer Product Review Sentiment tasks.}\label{fig:expr_scatter_bookcorpus}
    \end{subfigure}\vspace{-7pt}
    \caption{Metrics and performance of \imagenetsubset and \bookcorpus experiments. Each point represents a trained encoder, with its $x$- and $y$-coordinates showing $\lalign$ and $\lunif$ metrics and color showing the validation accuracy. \textbf{Blue} is better. Encoders with low $\lalign$ and $\lunif$ consistently perform well (lower left corners), even though the training methods (based on MoCo and Quick-Thought Vectors) are different from directly optimizing the contrastive loss in Equation~\eqref{eq:contrastive_loss}.}\label{fig:expr_scatter_extra}%
\end{figure*}

\paragraph{Both alignment and uniformity are necessary for a good representation. }
Figure~\ref{fig:expr_lgda_trend} shows how the final encoder changes in response to optimizing differently weighted combinations of $\lalign$ and $\lunif$ on \stl. The trade-off between the $\lalign$ and $\lunif$ indicates that perfect alignment and perfect uniformity are likely hard to simultaneously achieve in practice. However, the inverted-U-shaped accuracy curve confirms that both properties are indeed necessary for a good encoder.
When $\lalign$ is weighted much higher than $\lunif$, degenerate solution occurs and all inputs are mapped to the same feature vector ($\exp \lunif = 1)$. However, as long as the ratio between two weights is not too large (\eg, $< 4$), we observe that the representation quality remains relatively good and insensitive to the exact weight choices.

\paragraph{$\lalign$ and $\lunif$ causally affect downstream task performance.} We take an encoder trained with $\lcontr$ using a suboptimal temperature $\tau = 2.5$, and finetune it according to $\lalign$ and/or $\lunif$. Figure~\ref{fig:expr_finetune} visualizes the finetuning trajectories. When only one of alignment and uniformity is optimized, the corresponding metric improves, but both the other metric and performance degrade. However, when both properties are optimized, the representation quality steadily increases. These trends confirm the causal effect of alignment and uniformity on the representation quality, and suggest that directly optimizing them can be a reasonable choice.

\paragraph{Alignment and uniformity also matter in other contrastive representation learning variants.}
MoCo \citep{he2019momentum} and Quick-Thought Vectors \citep{logeswaran2018efficient} are contrastive representation learning variants that have nontrivial differences with directly optimizing $\lcontr$ in Equation~\eqref{eq:contrastive_loss}. MoCo introduces a memory queue and a momentum encoder. Quick-Thought Vectors uses two different encoders to encode each sentence in a positive pair, only normalizes encoder outputs during evaluation, and does not use random sampling to obtain minibatches. After modifying them to also allow $\lalign$ and $\lunif$, we train these methods on \imagenetsubset and \bookcorpus, respectively. Figure~\ref{fig:expr_scatter_extra} shows that $\lalign$ and $\lunif$ metrics are still correlated with the downstream task performances. Tables~\ref{tbl:expr_imagenet100}~and~\ref{tbl:expr_bookcorpus} show that directly optimizing them also leads to comparable or better representation quality. Table~\ref{tbl:expr_imagenet} also shows improvements on full \imagenet when we use $\lalign$ and $\lunif$ to train MoCo v2 \citep{chen2020improved} (an improved version of MoCo). These results suggest that alignment and uniformity are indeed desirable properties for representations, for \emph{both} image and text modalities, and are likely connected with general contrastive representation learning methods.

{
\begin{table*}[t!]%
    \centering

\vspace{6pt}%
\newcommand{\factor}{0.95}
\resizebox{
  0.9\width
}{!}{%
\small
\centering
\newcommand{\mround}[1]{\round{#1}{4}}%
\newcommand{\mroundprec}[1]{\round{#1}{2}\%}%
\renewcommand{\arraystretch}{1.4}%
\begin{tabular}{|c||c||c|c|}
    \hline
    &  \multirow{2}{*}{\vspace{-1pt}Loss Formula}
    &  \multicolumn{2}{c|}{Validation Set Accuracy $\uparrow$} \\
    \cline{3-4}
    &
    &  \hspace{8pt}\textrm{top1}\hspace{8pt}
    &  \hspace{8pt}\textrm{top5}\hspace{8pt} \\
    \hline\hline

    Best $\lcontr$ only
    & $\lcontr(\tau\narroweq0.07)$
    & \mroundprec{72.79999542236328}
    & \mroundprec{91.63999938964844}
    \\
    \hline

    Best $\lalign$ and $\lunif$ only
    & $3 \cdot \lalign(\alpha\narroweq2) + \lunif(t\narroweq3)$
    & \textbf{\mroundprec{74.5999984741211}}
    & \textbf{\mroundprec{92.73999786376953}}
    \\
    \hline

    Best among all encoders
    & $3 \cdot \lalign(\alpha\narroweq2) + \lunif(t\narroweq3)$
    & \textbf{\mroundprec{74.5999984741211}}
    & \textbf{\mroundprec{92.73999786376953}}
    \\



    \hline
\end{tabular}%
}%
\vspace{-5pt}%
\caption{\imagenetsubset encoder evaluations. Numbers show validation set accuracies of linear classifiers trained on encoder penultimate layer activations. The encoders are trained using MoCo-based methods. The best result is picked based on \textrm{top1} accuracy from a $3$-fold training set cross validation, among all $45$ encoders trained from scratch with $128$-dimensional output and $128$ batch size. }
\label{tbl:expr_imagenet100}

\end{table*}
\begin{table*}[t!]%
    \vspace{-1pt}
    \centering

\newcommand{\factor}{0.95}
\resizebox{
  0.9\width
}{!}{%
\small%
\centering%
\newcommand{\mround}[1]{\round{#1}{4}}%
\newcommand{\mroundprec}[1]{\round{#1}{2}\%}%
\renewcommand{\arraystretch}{1.4}%
\begin{tabular}{|c||c|c||c|c|}
    \hline
    &  \multicolumn{2}{c||}{\moviereview Classification}
    &  \multicolumn{2}{c|}{\customerreview Classification}
    \\
    \cline{2-5}
    &  \multirow{2}{*}{Loss Formula}
    &  \multirow{2}{*}{\vspace{-2pt}\shortstack{Val.~Set\\Accuracy $\uparrow$}}
    &  \multirow{2}{*}{Loss Formula}
    &  \multirow{2}{*}{\vspace{-2pt}\shortstack{Val.~Set\\Accuracy $\uparrow$}}
    \\
    & & & &
    \\
    \hline\hline

    Best $\lcontr$ only
    & $\lcontr(\tau\narroweq0.075)$
    & \textbf{\mroundprec{77.50702905342081}}
    & $\lcontr(\tau\narroweq0.05)$
    & \textbf{\mroundprec{83.86243386243386}}
    \\
    \hline

    Best $\lalign$ and $\lunif$ only
    & $0.9 \cdot \lalign(\alpha\narroweq2) + 0.1 \cdot \lunif(t\narroweq5)$
    & \mroundprec{73.75820056232428}
    & $0.9 \cdot \lalign(\alpha\narroweq2) + 0.1 \cdot \lunif(t\narroweq5)$
    & \mroundprec{80.95238095238095}
    \\
    \hline

    Best among all encoders
    & $\lcontr(\tau\narroweq0.075)$
    & \textbf{\mroundprec{77.50702905342081}}
    & $\lcontr(\tau\narroweq0.05)$
    & \textbf{\mroundprec{83.86243386243386}}
    \\



    \hline
\end{tabular}%
}%
\vspace{-5pt}%
\caption{%
\bookcorpus encoder evaluations. Numbers show Movie Review Sentence Polarity (\moviereview) and Customer Product Sentiment (\customerreview) validation set classification accuracies of logistic classifiers fit on encoder outputs. The encoders are trained using Quick-Thought-Vectors-based methods. The best result is picked based on accuracy from a $5$-fold training set cross validation, individually for \moviereview and \customerreview, among all $108$ encoders trained from scratch with $1200$-dimensional output and $400$ batch size.}
\label{tbl:expr_bookcorpus}

    \vspace{-7pt}
\end{table*}
}

\begin{table}[h]%
    \centering

\newcommand{\factor}{0.95}
\resizebox{
  0.9\width
}{!}{%
\small
\centering
\newcommand{\mround}[1]{\round{#1}{4}}%
\newcommand{\mroundprec}[1]{\round{#1}{2}\%}%
\renewcommand{\arraystretch}{1.4}%
\begin{tabular}{|c|c|}
    \hline
    Loss Formula
    & Validation Set \textrm{top1} Accuracy $\uparrow$ \\
    \hline\hline

    \multirow{2}{*}{\shortstack{$\lcontr(\tau\narroweq0.2)$\\(MoCo v2 \citet{chen2020improved})}}
    & \multirow{2}{*}{$67.5\% \pm 0.1\%$}
    \\
    & \\
    \hline

    $3 \cdot \lalign(\alpha\narroweq2) + \lunif(t\narroweq3)$
    & \textbf{\mroundprec{67.69400024414062}}
    \\
    \hline
\end{tabular}%
}%
\vspace{-5pt}%
\caption{\imagenet encoder evaluations with MoCo v2, and its variant with $\lalign$ and $\lunif$. MoCo v2 results are from the MoCo v2 official implementation \citep{chen2020mocov2github}, with mean and standard deviation across $5$ runs. Both settings use $200$ epochs of unsupervised training.}
\label{tbl:expr_imagenet}

    \vspace{-5pt}
\end{table}

\section{Discussion}

\emph{Alignment} and \emph{uniformity} are often alluded to as motivations for representation learning methods (see Figure~\ref{fig:align_unif}). However, a thorough understanding of these properties is lacking in the literature.

Are they in fact related to the representation learning methods? Do they actually agree with the representation quality (measured by downstream task performance)?

In this work, we have presented a detailed investigation on the relation between these properties and the popular paradigm of contrastive representation learning. Through theoretical analysis and extensive experiments, we are able to relate the contrastive loss with the alignment and uniformity properties, and confirm their strong connection with downstream task performances. Remarkably, we have revealed that directly optimizing our proposed metrics often leads to representations of better quality.

Below we summarize several suggestions for future work.

\paragraph{Niceness of the unit hypersphere.} Our analysis was based on the empirical observation that representations are often $\ell_2$ normalized. Existing works have motivated this choice from a manifold mapping perspective \citep{liu2017sphereface,s-vae18} and computation stability \citep{xu2018spherical,wang2017normface}. However, to our best knowledge, the question of why the unit hypersphere is a nice feature space is not yet rigorously answered. One possible direction is to formalize the intuition that connected sets with smooth boundaries are nearly linearly separable in the hyperspherical geometry (see Figure~\ref{fig:hypersphere_linsep}), since linear separability is one of the most widely used criteria for representation quality and is related to the notion of disentanglement \citep{higgins2018towards}.

\paragraph{Beyond contrastive learning.} Our analysis focused on the relationship between contrastive learning and the alignment and uniformity properties on the unit hypersphere. However, the ubiquitous presence of $\ell_2$ normalization in the representation learning literature suggests that the connection may be more general. In fact, several existing empirical methods are directly related to uniformity on the hypersphere \citep{bojanowski2017unsupervised,s-vae18,xu2018spherical}. We believe that relating a broader class of representations to uniformity and/or alignment on the hypersphere will provide novel insights and lead to better empirical algorithms.

\clearpage
\newpage
\section*{Acknowledgements}
We thank Philip Bachman, Ching-Yao Chuang, Justin Solomon, Yonglong Tian, and Zhenyang Zhang for many helpful comments and suggestions. Tongzhou Wang was supported by the MIT EECS Merrill Lynch Graduate Fellowship. We thank Yangjun Ruan for helping us realize a minor issue with STL-10 scatter plot (\Cref{fig:expr_scatter}, now fixed).

\section*{Major Changelog}
\paragraph{8/24/2020:}
\begin{itemize}[leftmargin=12pt]
    \item Added results on full ImageNet and MoCo v2.
\end{itemize}
\paragraph{11/6/2020:}
\begin{itemize}[leftmargin=12pt]
    \item Added discussions on the range of $\lunif$.
    \item Corrected Theorem~\ref{thm:asym_inf_negatives}'s convergence rate to $\mathcal{O}(M^{-1/2})$.
\end{itemize}
\paragraph{8/15/2022:}
\begin{itemize}[leftmargin=12pt]
    \item Removed from \Cref{fig:expr_scatter,supp:tbl:stl10-big} two STL-10 encoders that should not be included due to their usage of other regularizers (not shown). This does not affect the observed relation among $\lalign$, $\lunif$, and downstream performance. All other text and discussions stay unchanged.
\end{itemize}
{\small
\bibliography{reference}
\bibliographystyle{icml2020}
}

\onecolumn
\appendix

\section{Proofs and Additional Theoretical Analysis}

In this section, we present proofs for propositions and theorems in main paper \Cref{sec:unif,sec:limit}.

The propositions in \Cref{sec:unif} illustrate the deep relations between the Gaussian kernel $G_t \colon \sphere^d \times \sphere^d \rightarrow \R$ and the uniform distribution on the unit hypersphere $\sphere^d$. As we will show below in \Cref{supp:sec:proof-prop-gauss-unif}, these properties directly follow well-known results on strictly positive definite kernels.

In \Cref{supp:sec:proofs-asymptotics}, we present a proof for \Cref{thm:asym_inf_negatives}. \Cref{thm:asym_inf_negatives} describes the asymptotic behavior of $\lcontr$ as the number of negative samples $M$ approaches infinity. The theorem is strongly related to empirical contrastive learning, given an error term (deviation from the limit) decaying in $\mathcal{O}(M^{-1/2})$ and that empirical practices often use a large number of negatives (\eg, $M=65536$ in \citet{he2019momentum}) based on the observation that using more negatives consistently leads to better representation quality \citep{wu2018unsupervised,tian2019contrastive,he2019momentum}. Our proof further reveals connections between $\lcontr$ and $\lunif$ which is defined via the Gaussian kernel.

Finally, also in \Cref{supp:sec:proofs-asymptotics}, we present a weaker result on the setting where only a single negative is used in $\lcontr$ (\ie, $M=1$).

\subsection{Proofs for \Cref{sec:unif} and Properties of $\lunif$} \label{supp:sec:proof-prop-gauss-unif}

To prove \Cref{prop:continuous-problem}~and~\ref{prop:discrete-problems}, we utilize the \emph{strict positive definiteness} \citep{bochner1992monotone,stewart1976positive} of the Gaussian kernel $G_t$: \begin{equation*}
    G_t(u, v) \trieq e^{-t \norm{u - v}_2^2} = e^{2t \cdot u\T v - 2t}, \quad t > 0.
\end{equation*} From there, we apply a known result about such kernels, from which the two propositions directly follow.

\begin{definition}[Strict positive definiteness \citep{bochner1992monotone,stewart1976positive}]
    A symmetric and lower semi-continuous kernel $K$ on $A \times A$ (where $A$ is infinite and compact) is called strictly positive definite if for every finite signed Borel measure $\mu$ supported on $A$ whose energy \begin{equation*}
        I_K[\mu] \trieq \int_{\sphere^d} \int_{\sphere^d} K(u, v) \diff \mu(v) \diff \mu(u)
    \end{equation*}
    is well defined, we have $I_K[\mu] \geq 0$, where equality holds only if $\mu \equiv 0$ on the $\sigma$-algebra of Borel subsets of $A$.
\end{definition}

\begin{definition}
    Let $\mathcal{M}(\sphere^{d})$ be the set of Borel probability measures on $\sphere^{d}$.
\end{definition}

We are now in the place to apply the following two well-known results, which we present by restating Proposition~4.4.1, Theorem~6.2.1 and Corollary~6.2.2 of \citet{borodachov2019discrete} in weaker forms. We refer readers to \citet{borodachov2019discrete} for their proofs.
\begin{lemma}[Strict positive definiteness of $G_t$] \label{supp:lemma-strict-pd-gaussian}
    For $t > 0$, the Gaussian kernel $G_t(u, v) \trieq e^{-t \norm{u - v}_2^2} = e^{2t \cdot u\T v - 2t}$ is strictly positive definite  on $\sphere^d \times \sphere^d$.
\end{lemma}

\begin{lemma}[Strictly positive definite kernels on $\sphere^d$] \label{supp:lemma:strict-pd-uniform}
    Consider kernel $K_f \colon \sphere^d \times \sphere^d \rightarrow (-\infty, +\infty]$ of the form, \begin{equation}
        K_f(u, v) \trieq f(\norm{u - v}_2^2).
    \end{equation}
    If $K_f$ is strictly positive definite on $\sphere^d \times \sphere^d $ and $I_{K_f}[\sigma_d]$ is finite, then $\sigma_d$ is the unique measure (on Borel subsets of $\sphere^d$) in the solution of $\min_{\mu \in \mathcal{M}(\sphere^d)} I_{K_f}[\mu]$, and the normalized counting measures associated with any $K_f$-energy minimizing sequence of $N$-point configurations on $\sphere^d$ converges weak$^*$ to $\sigma_d$.

    In particular, this conclusion holds whenever $f$ has the property that $-f'(t)$ is strictly completely monotone on $(0, 4]$ and $I_{K_f}[\sigma_d]$ is finite.
\end{lemma}

We now recall Propositions~\ref{prop:continuous-problem}~and~\ref{prop:discrete-problems}.
\begingroup
\def\theproposition{\ref{prop:continuous-problem}}
\begin{proposition}
    $\sigma_{d}$ is the unique solution (on Borel subsets of $\sphere^d$) of \begin{equation}
        \min_{\mu \in \mathcal{M}(\sphere^{d})} I_{G_t}[\mu] = \min_{\mu \in \mathcal{M}(\sphere^{d})} \int_{\sphere^d} \int_{\sphere^d} G_t(u, v) \diff \mu(v) \diff \mu(u). \label{supp:prop:continuous-problem-solution}
    \end{equation}
\end{proposition}
\addtocounter{proposition}{-1}
\endgroup
\begin{proof}[Proof of \Cref{prop:continuous-problem}]
    This is a direct consequence of \Cref{supp:lemma-strict-pd-gaussian,supp:lemma:strict-pd-uniform}.
\end{proof}

\begingroup
\def\theproposition{\ref{prop:discrete-problems}}
\begin{proposition}
    For each $N > 0$, the $N$ point minimizer of the average pairwise potential is \begin{equation*}
        \mathbf{u}^*_N = \argmin_{u_1, u_2, \dots, u_N \in \sphere^{d}} \sum_{1 \leq i < j \leq N} G_t(u_i, u_j).
    \end{equation*}
    The normalized counting measures associated with the $\{\mathbf{u}^*_N\}_{N=1}^\infty$ sequence converge weak$^*$ to $\sigma_d$.
\end{proposition}
\addtocounter{proposition}{-1}
\endgroup
\begin{proof}[Proof of \Cref{prop:discrete-problems}]
    This is a direct consequence of \Cref{supp:lemma-strict-pd-gaussian,supp:lemma:strict-pd-uniform}.
\end{proof}

\subsubsection{More Properties of $\lunif$}

\paragraph{Range of $\lunif$.} It's not obvious what the optimal value of $\lunif$ is. In the following proposition, we characterize the exact range of the expected Gaussian potential and how it evolves as dimensionality increases. The situation for $\lunif$ directly follows as a corollary.

\begin{proposition}[Range of the expected pairwise Gaussian potential $G_t$]\label{supp:prop:gaussian-range}
    For $t > 0$, the expected pairwise Gaussian potential \wrt Borel probability measure $\mu \in \mathcal{M}(\sphere^d)$
    \begin{equation*}
        I_{G_t}[\mu] = \int_{\sphere^d} \int_{\sphere^d} G_t(u, v) \diff \mu(v) \diff \mu(u)
    \end{equation*}has range $[e^{-2 t} \prescript{}{0}{F}_1(;\frac{d+1}{2}; t^2), 1]$, where $\prescript{}{0}{F}_1$ is the confluent hypergeometric limit function defined as \begin{equation}
        \prescript{}{0}{F}_1(;\alpha;z) \trieq \sum_{n=0}^\infty \frac{z^n}{(\alpha)_n n!}, \label{supp:eq:0f1-formula}
    \end{equation}
    where we have used the Pochhammer symbol $(a)_n = \begin{cases}
        1 & \mbox{if } n = 0 \\
        a(a+1)(n+2) \dots (a+n-1) & \mbox{if } n \geq 1.
    \end{cases}$

    We have \begin{itemize}
        \item The minimum $e^{-2 t} \prescript{}{0}{F}_1(;\frac{d+1}{2}; t^2)$ is achieved iff $\mu = \sigma_d$ (on Borel subsets of $\sphere^d$). Furthermore, this value strictly decreases as $d$ increases, converging to $e^{-2t}$ in the limit of $d \rightarrow \infty$.
        \item The maximum is achieved iff $\mu$ is a Dirac delta distribution, \ie, $\mu = \delta_u$  (on Borel subsets of $\sphere^d$), for some $u \in \sphere^d$.
    \end{itemize}
\end{proposition}
\begin{proof}[Proof of \Cref{supp:prop:gaussian-range}]~
    \begin{itemize}
        \item \itempara{Minimum. }

        We know from \Cref{prop:continuous-problem} that $\sigma_d$ {\em uniquely} achieves the minimum, given by the following integral ratio \begin{align*}
            I_{G_t}[\sigma_d]
            & = \frac{\int_0^\pi e^{-t (2 \sin \frac{\theta}{2})^2} \sin^{d-1} \theta \diff \theta}{\int_0^\pi \sin^{d-1} \theta \diff \theta} \\
            & = \frac{\int_0^\pi e^{-2t (1 - \cos \theta)} \sin^{d-1} \theta \diff \theta}{\int_0^\pi \sin^{d-1} \theta \diff \theta} \\
            & = e^{-2t} \frac{\int_0^\pi e^{2t\cos \theta} \sin^{d-1} \theta \diff \theta}{\int_0^\pi \sin^{d-1} \theta \diff \theta}.
        \end{align*}

        The denominator, with some trigonometric identities, can be more straightforwardly evaluated as \begin{equation*}
            \int_0^\pi \sin^{d-1} \theta \diff \theta = \sqrt{\pi} \frac{\Gamma(\frac{d}{2})}{\Gamma(\frac{d+1}{2})}.
        \end{equation*}

        The numerator is \begin{align*}
            \int_0^\pi e^{2t\cos \theta} \sin^{d-1} \theta \diff \theta
            & =
            - \int_0^\pi e^{2t\cos \theta} \sin^{d-2} \theta \cos' \theta \diff \theta \\
            & =
            \int_{-1}^1 e^{2t s} (1-s^2)^{d/2-1} \diff s \\
            & = \frac{\Gamma(\frac{d-1}{2} + \frac{1}{2}) \sqrt{\pi}}{\Gamma(\frac{d-1}{2}+1)} \prescript{}{0}{F}_1(;\frac{d-1}{2}+1;-\frac{1}{4}(-2it)^2) \\
            & = \frac{\Gamma(\frac{d}{2}) \sqrt{\pi}}{\Gamma(\frac{d+1}{2})} \prescript{}{0}{F}_1(;\frac{d+1}{2};t^2),
        \end{align*}
        where we have used the following identity based on the Poisson formula for Bessel functions and the relationship between $\prescript{}{0}{F}_1$ and Bessel functions: \begin{equation*}
            \int_{-1}^1 e^{iz s} (1-s^2)^{\nu - \frac{1}{2}} \diff s =
            \frac{\Gamma(\nu + \frac{1}{2}) \sqrt{\pi}}{(\frac{z}{2})^\nu}  J_\nu(z) =
            \frac{\Gamma(\nu + \frac{1}{2}) \sqrt{\pi}}{\Gamma(\nu+1)} \prescript{}{0}{F}_1(;\nu+1;-\frac{1}{4}z^2).
        \end{equation*}

        Putting both together, we have \begin{align*}
            I_{G_t}[\sigma_d]
            & = e^{-2t} \frac{\int_0^\pi e^{2t\cos \theta} \sin^{d-1} \theta \diff \theta}{\int_0^\pi \sin^{d-1} \theta \diff \theta} \\
            & = e^{-2t} \dfrac{\frac{\Gamma(\frac{d}{2}) \sqrt{\pi}}{\Gamma(\frac{d+1}{2})} \prescript{}{0}{F}_1(;\frac{d+1}{2};t^2)}{\sqrt{\pi} \frac{\Gamma(\frac{d}{2})}{\Gamma(\frac{d+1}{2})}} \\
            & = e^{-2t} \prescript{}{0}{F}_1(;\frac{d+1}{2};t^2) \\
            & = e^{-2t} \sum_{n=0}^\infty \frac{t^{2n}}{(\frac{d+1}{2})_n n!},
        \end{align*}
        where we have used the definition of $\prescript{}{0}{F}_1$ in \Cref{supp:eq:0f1-formula} to expand the formula.

        Notice that each summand strictly decreases as $d \rightarrow \infty$. So must the total sum.

        For the asymptotic behavior at $d \rightarrow \infty$, it only remains to show that  \begin{equation}
            \lim_{d \rightarrow \infty} \sum_{n=0}^\infty \frac{t^{2n}}{(\frac{d+1}{2})_n n!} = 1. \label{supp:eq:0f1-convergence-series-sum}
        \end{equation}

        For the purpose of applying the Dominated Convergence Theorem (DCT) (on the counting measure). We consider the following summable series \begin{equation*}
            \sum_{n=0}^\infty \frac{t^{2n}}{n!} = e^{t^2},
        \end{equation*}
        with each term bounding the corresponding one in \Cref{supp:eq:0f1-convergence-series-sum}: \begin{equation*}
            \frac{t^{2n}}{n!} \geq \frac{t^{2n}}{(\frac{d+1}{2})_n n!},\qquad\qquad \forall n \geq 0, d > 0.
        \end{equation*}

        Thus, \begin{equation*}
            \lim_{d \rightarrow \infty} \sum_{n=0}^\infty \frac{t^{2n}}{(\frac{d+1}{2})_n n!} = \sum_{n=0}^\infty \lim_{d \rightarrow \infty}  \frac{t^{2n}}{(\frac{d+1}{2})_n n!} = 1 + 0 + 0 + \dots  = 1.
        \end{equation*}

        Hence, the asymptotic lower range is $e^{-2t}$.

        \item \itempara{Maximum.}

        Obviously, Dirac delta distributions $\delta_u$, $u \in \sphere^d$ would achieve a maximum of $1$. We will now show that all Borel probability measures $\mu$ \st $I_{G_t}[\mu] = 1$ are delta distributions.

        Suppose that such a $\mu$ is not a Dirac delta distribution. Then, we can take distinct $x, y \in \supp(\mu) \subseteq \sphere^d$, and open neighborhoods around $x$ and $v$, $N_x, N_y \in \sphere^d$ such that they are small enough and disjoint: \begin{align*}
            N_x & \trieq \{u \in \sphere^d \colon \norm{u - x}_2 < \frac{1}{3} \norm{x - y}_2 \} \\
            N_y & \trieq \{u \in \sphere^d \colon \norm{u - y}_2 < \frac{1}{3} \norm{x - y}_2 \}.
        \end{align*}
        Then, \begin{align*}
            I_{G_t}[\mu]
            & = \int_{\sphere^d} \int_{\sphere^d} G_t(u, v) \diff \mu(v) \diff \mu(u) \\
            & = \int_{\sphere^d} \int_{\sphere^d} e^{-t \norm{u-v}_2^2} \diff \mu(v) \diff \mu(u) \\
            & \leq (1 - 2 \mu({N_x}) \mu({N_y})) e^{-t \cdot 0} + 2 \int_{N_x} \int_{N_y} e^{-t \norm{u-v}_2^2} \diff \mu(v) \diff \mu(u) \\
            & < 1 - 2 \mu({N_x}) \mu({N_y}) + 2 \mu({N_x}) \mu({N_y}) e^{-t(\norm{x - y}_2/3)^2} \\
            & = 1 - 2 \mu({N_x}) \mu({N_y}) (1 - e^{-\frac{t}{9}\norm{x - y}_2^2}) \\
            & < 1.
        \end{align*}

        Hence, only Dirac delta distributions attain the maximum.
    \end{itemize}

\end{proof}


\begin{corollary}[Range of $\lunif$]\label{supp:coro:range-lunif}
    For encoder $f \colon \R^n \rightarrow \sphere^{m-1}$, $\lunif(f; t) \in [-2t + \log \prescript{}{0}{F}_1(;\frac{m}{2};t^2), 0]$, where the lower bound $-2t + \log \prescript{}{0}{F}_1(;\frac{m}{2};t^2)$ is achieved only by perfectly uniform encoders $f$, and the upper bound $0$ is achieved only by degenerate encoders that output a fixed feature vector almost surely.

    Furthermore, the lower bound strictly decreases as the output dimension $m$ increases, attaining the following asymptotic value \begin{equation}
        \lim_{m\rightarrow\infty} -2t+\log \prescript{}{0}{F}_1(;\frac{m}{2};t^2) = -2t.\label{supp:eq:range-lunif-lower-bound-asymptotics}
    \end{equation}
\end{corollary}


\begin{figure}[t]\vspace{-3pt}
    \centering
    \begin{minipage}[t]{0.484\linewidth}
        \includegraphics[width=\linewidth]{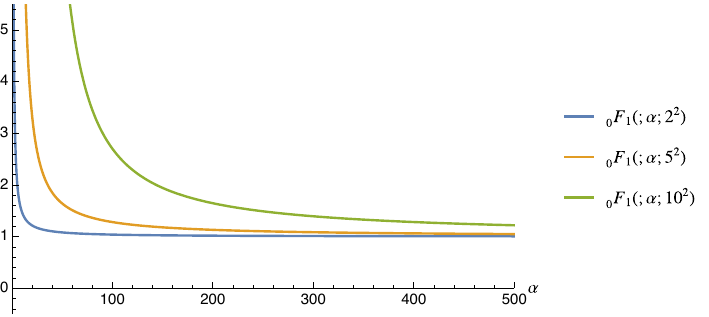}
        \caption{Asymptotic behavior of $\prescript{}{0}{F}_1(;\alpha;z)$. For $z>0$, as $\alpha$ grows larger, the function converges to $1$.}
        \label{supp:fig:0f1-asymptotics}
    \end{minipage}\hfill
    \begin{minipage}[t]{0.484\linewidth}
        \includegraphics[width=\linewidth]{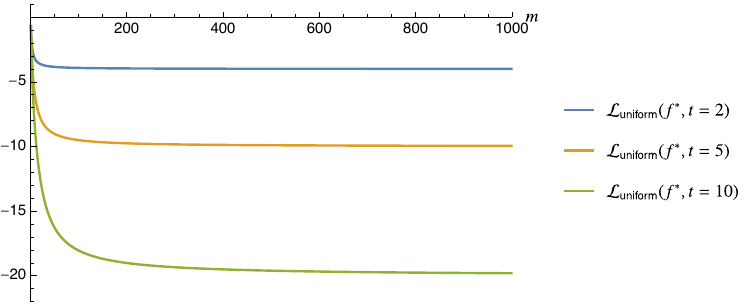}
        \caption{Asymptotic behavior of optimal $\lunif(f, t)$, attained by a perfectly uniform encoder $f^*$. As the feature dimension $m$ grows larger, the value converges to $-2t$.}
        \label{supp:fig:lunif-optima-asymptotics}
    \end{minipage}
\end{figure}

\paragraph{Intuition for the optimal $\lunif$ value in high dimensions.} If we ignore the $\log \prescript{}{0}{F}_1(;\frac{m}{2};t^2)$ term, informally, the optimal value of $-2t$ roughly says that any pair of feature vectors on $\sphere^{d}$ has distance about $\sqrt{2}$, \ie, are nearly orthogonal to each other. Indeed, vectors of high dimensions are usually nearly orthogonal, which is also consistent with the asymptotic result in \Cref{supp:eq:range-lunif-lower-bound-asymptotics}.

\Cref{supp:fig:0f1-asymptotics,supp:fig:lunif-optima-asymptotics} visualize how $\prescript{}{0}{F}_1$ and the optimal $\lunif$ (given by perfectly uniform encoders) evolve.

\paragraph{Lower bound of $\lunif$ estimates.} In practice, when $\lunif$ calculated using expectation over (a batch of) empirical samples $\{x_i\}_{i=1}^B$, $B>1$, the range in Corollary~\ref{supp:coro:range-lunif} is indeed valid, since it bounds over all distributions: \begin{equation}
    \hat{\mathcal{L}}^{(1)}_\mathsf{uniform} \trieq \log \frac{1}{B^2} \sum_{i=1}^B\sum_{j=1}^B e^{-t\norm{f(x_i) - f(x_j)}^2} > -2t + \log \prescript{}{0}{F}_1(;\frac{m}{2};t^2).\label{supp:eq:empirical-cdist-lunif-follows-range}
\end{equation}However, often $\lunif$ is empirically estimated without considering distances between a vector and itself (\eg, in \Cref{fig:pytorch-losses-code} and in our experiment settings as described in \Suppmatsection~\ref{supp:sec:expr-details}):\begin{equation}
    \hat{\mathcal{L}}^{(2)}_\mathsf{uniform} \trieq \log \frac{1}{B(B-1)} \sum_{i=1}^B\sum_{j\in\{1, \dots, B\} \setminus \{i\}} e^{-t\norm{f(x_i) - f(x_j)}^2}.\label{supp:eq:empirical-pdist-lunif}
\end{equation}While both quantities converge to the correct value in the limit, the lower bound is not always true for this one, because it is {\em not} the expected pairwise Gaussian kernel based on some distribution. Note the following relation:
\begin{equation*}
    \hat{\mathcal{L}}^{(2)}_\mathsf{uniform} = \log \left(\frac{B\cdot\exp( \hat{\mathcal{L}}^{(1)}_\mathsf{uniform} ) - 1}{B-1} \right).
\end{equation*}
We can derive a valid lower bound using \Cref{supp:eq:empirical-cdist-lunif-follows-range}: for $\prescript{}{0}{F}_1(;\frac{m}{2};t^2) > \frac{e^{2t}}{B}$, \begin{equation*}
    \hat{\mathcal{L}}^{(2)}_\mathsf{uniform} > \log \left(\frac{B\cdot\exp( -2t + \log \prescript{}{0}{F}_1(;\frac{m}{2};t^2) ) - 1}{B-1} \right) = \log \left(\frac{B e^{-2t} \prescript{}{0}{F}_1(;\frac{m}{2};t^2) - 1}{B-1} \right).
\end{equation*}
Since this approaches fails for cases that $\prescript{}{0}{F}_1(;\frac{m}{2};t^2) \leq \frac{e^{2t}}{B}$, we can combine it with the naive lower bound $-4t$, and have \begin{equation*}
    \hat{\mathcal{L}}^{(2)}_\mathsf{uniform} > \begin{cases}
        \max(-4t, \log \left(\frac{B e^{-2t} \prescript{}{0}{F}_1(;\frac{m}{2};t^2) - 1}{B-1} \right)) & \mbox{if } \prescript{}{0}{F}_1(;\frac{m}{2};t^2) > \frac{e^{2t}}{B} \\
        -4t & \mbox{otherwise.}
    \end{cases}
\end{equation*}

\paragraph{Non-negative versions of $\lunif$ for practical uses.} By definition, $\lunif$ always non-positive. As shown above, different $\lunif$ empirical estimates may admit different lower bounds. However, in our experience, for reasonably large batch sizes, adding an offset of $2t$ often ensures a non-negative loss that is near zero at optimum. When output dimensionality $m$ is low, it might be useful to add an additional offset of $-\log \prescript{}{0}{F}_1(;\frac{m}{2};t^2)$, which can be computed with the help of the SciPy package function \texttt{scipy.special.hyp0f1(m/2, t**2)} \citep{virtanen2020scipy}.

\subsection{Proofs and Additional Results for Section~4.2} \label{supp:sec:proofs-asymptotics}

The following lemma directly follows Theorem~3.3 and Remarks~3.4~(b)(i) of \citet{serfozo1982convergence}. We refer readers to \citet{serfozo1982convergence} for its proof.

\begin{lemma} \label{supp:lemma:int-conv-fn-conv-measure}
    Let $A$ be a compact second countable Hausdorff space. Suppose \begin{enumerate}
        \item $\{\mu_n\}_{n=1}^\infty$ is a sequence of finite and positive Borel measures supported on $A$ that converges weak$^*$ to some finite and positive Borel measure $\mu$ (which is same as vague convergence since $A$ is compact);
        \item $\{f_n\}_{n=1}^\infty$ is a sequence of Borel measurable functions that converges continuously to a Borel measurable $f$;
        \item $\{f_n\}_n$ are uniformly bounded over $A$.
    \end{enumerate}
    Then, we have the following convergence: \begin{equation*}
        \lim_{n \rightarrow \infty} \int_{x \in A} f_n(x) \diff \mu_n(x) = \int_{x \in A} f(x) \diff \mu(x).
    \end{equation*}
\end{lemma}

We now recall \Cref{thm:asym_inf_negatives}.
\begingroup
\def\thetheorem{\ref{thm:asym_inf_negatives}}
\begin{theorem}[Asymptotics of $\lcontr$]
    For fixed $\tau > 0$, as the number of negative samples $M \rightarrow \infty$, the (normalized) contrastive loss converges to \begin{align}
        & \lim_{M \rightarrow \infty} \lcontr(f; \tau, M) - \log M \notag \\
        & \qquad\qquad =
        \lim_{M \rightarrow \infty} \expectunder[\substack{
            (x, y) \sim \distnpos \\
            \{x^-_i\}_{i=1}^M \iidsim \distndata
        }]{- \log \frac{e^{f(x)\T f(y) / \tau}}{e^{f(x)\T f(y) / \tau} + \sum_i e^{f(x^-_i)\T f(y) / \tau}}} - \log M \notag \\
        & \qquad\qquad =
        -\frac{1}{\tau} \expectunder[(x, y) \sim \distnpos]{f(x)\T f(y)} + \expectunder[x \sim \distndata]{\log\expectunder[x^- \sim \distndata]{e^{f(x^-)\T f(x) / \tau}}}.
        \tag{\ref{eq:contrastive_loss_limit}}
    \end{align}

    We have the following results: \begin{enumerate}
        \item \label{thm:asym_inf_negatives:itm:align} The first term is minimized iff $f$ is perfectly aligned.
        \item \label{thm:asym_inf_negatives:itm:uniform} If perfectly uniform encoders exist, they form the exact minimizers of the second term.
        \item \label{thm:asym_inf_negatives:itm:convergence-rate} For the convergence in \Cref{eq:contrastive_loss_limit}, the absolute deviation from the limit (\ie, the error term) decays in $\mathcal{O}(M^{-1/2})$.
    \end{enumerate}
\end{theorem}
\addtocounter{theorem}{-1}
\endgroup

\begin{proof}[Proof of \Cref{thm:asym_inf_negatives}]
We first show the convergence stated in \Cref{eq:contrastive_loss_limit} along with its speed (result~\ref{thm:asym_inf_negatives:itm:convergence-rate}), and then the relations between the two limiting terms and the alignment and uniformity properties (results~\ref{thm:asym_inf_negatives:itm:align}~and~\ref{thm:asym_inf_negatives:itm:uniform}).

\begin{itemize}
    \item \itempara{Proof of the convergence in \Cref{eq:contrastive_loss_limit} and the $\mathcal{O}(M^{-1/2})$ decay rate of its error term (result~\ref{thm:asym_inf_negatives:itm:convergence-rate}).}

    Note that for any $x, y \in \R^n$ and $\{x^-_i\}_{i=1}^{M} \iidsim \pdata$, we have \begin{equation}
        \lim_{M \rightarrow \infty} \log\left(\frac{1}{M} e^{f(x)\T f(y) / \tau} + \frac{1}{M} \sum_{i=1}^M e^{f(x^-_i)\T f(x) / \tau} \right) = \log\expectunder[x^- \sim \pdata]{e^{f(x^-)\T f(x) / \tau}} \qquad\text{almost surely}, \label{supp:eq:asymp-inner-convergence-slln}
    \end{equation}
    by the strong law of large numbers (SLLN) and the Continuous Mapping Theorem.

    Then, we can derive
    \begin{equation*}
    \begin{split}
        & \lim_{M \rightarrow \infty} \lcontr(f; \tau, M) - \log M \\
        & \qquad\qquad = \expectunder[(x, y) \sim \ppos]{- f(x)\T f(y) / \tau} + \lim_{M \rightarrow \infty} \expectunder[\substack{
            (x, y) \sim \ppos \\
            \{x^-_i\}_{i=1}^{M} \iidsim \pdata
            }]{\log\left(\frac{1}{M} e^{f(x)\T f(y) / \tau} + \frac{1}{M} \sum_{i=1}^M e^{f(x^-_i)\T f(x) / \tau} \right)} \\
        & \qquad\qquad = \expectunder[(x, y) \sim \ppos]{- f(x)\T f(y) / \tau} + \expect{\lim_{M \rightarrow \infty}  \log\left(\frac{1}{M} e^{f(x)\T f(y) / \tau} + \frac{1}{M} \sum_{i=1}^M e^{f(x^-_i)\T f(x) / \tau} \right)}\\
        & \qquad\qquad = -\frac{1}{\tau} \expectunder[(x, y) \sim \ppos]{f(x)\T f(y)} + \expectunder[x \sim \pdata]{\log\expectunder[x^- \sim \pdata]{e^{f(x^-)\T f(x) / \tau}}},
    \end{split}
    \end{equation*}
    where we justify the switching of expectation and limit by the convergence stated in \Cref{supp:eq:asymp-inner-convergence-slln}, the boundedness of $e^{u\T v/\tau}$ (where $u, v \in \sphere^d, \tau > 0$), and the Dominated Convergence Theorem (DCT).\par

    For convergence speed, we have \begin{align}
        & \abs{\left( \lim_{M \rightarrow \infty} \lcontr(f; \tau, M) - \log M \right) - \left(\lcontr(f; \tau, M) - \log M\right)} \notag\\
        & \qquad\qquad = \abs{\expectunder[\substack{
            (x, y) \sim \ppos \\
            \{x^-_i\}_{i=1}^{M} \iidsim \pdata
            }]{\log\expectunder[x^- \sim \pdata]{e^{f(x^-)\T f(x) / \tau}} - \log\left(\frac{1}{M} e^{f(x)\T f(y) / \tau} + \frac{1}{M} \sum_{i=1}^M e^{f(x^-_i)\T f(x) / \tau}\right)}} \notag\\
        & \qquad\qquad \leq \expectunder[\substack{
            (x, y) \sim \ppos \\
            \{x^-_i\}_{i=1}^{M} \iidsim \pdata
            }]{\abs{\log\expectunder[x^- \sim \pdata]{e^{f(x^-)\T f(x) / \tau}} - \log\left(\frac{1}{M} e^{f(x)\T f(y) / \tau} + \frac{1}{M} \sum_{i=1}^M e^{f(x^-_i)\T f(x) / \tau}\right)}} \notag\\
        & \qquad\qquad \leq e^{1/\tau} \expectunder[\substack{
            (x, y) \sim \ppos \\
            \{x^-_i\}_{i=1}^{M} \iidsim \pdata
            }]{\abs{\expectunder[x^- \sim \pdata]{e^{f(x^-)\T f(x) / \tau}} - \left(\frac{1}{M} e^{f(x)\T f(y) / \tau} + \frac{1}{M} \sum_{i=1}^M e^{f(x^-_i)\T f(x) / \tau}\right)}} \notag\\
        & \qquad\qquad \leq \frac{1}{M} e^{2/\tau} + e^{1/\tau} \expectunder[
            x, \{x^-_i\}_{i=1}^{M} \iidsim \pdata]{\abs{\expectunder[x^- \sim \pdata]{e^{f(x^-)\T f(x) / \tau}} - \frac{1}{M} \sum_{i=1}^M e^{f(x^-_i)\T f(x) / \tau}}} \notag\\
        & \qquad\qquad = \frac{1}{M} e^{2/\tau} + \mathcal{O}(M^{-1/2}),
    \end{align}
    where the first inequality follows the Intermediate Value Theorem and the $e^{1/\tau}$ upper bound on the absolute derivative of $\log$ between the two points, and the last equality follows the Berry-Esseen Theorem given the bounded support of $e^{f(x^-_i)\T f(x) / \tau}$ as following: for \iid random variables $Y_i$ with bounded support $\subset [-a, a]$, zero mean and $\sigma^2_Y \leq a^2$ variance, we have \begin{align*}
        \expect{\abs{\frac{1}{M} \sum_{i=1}^M Y_i}}
        & = \frac{\sigma_Y}{\sqrt{M}} \expect{\abs{\frac{1}{\sqrt{M}\sigma_Y} \sum_{i=1}^M Y_i}} \\
        & = \frac{\sigma_Y}{\sqrt{M}} \int_0^{\frac{a\sqrt{M}}{\sigma_Y}} \prob{\abs{\frac{1}{\sqrt{M}\sigma_Y} \sum_{i=1}^M Y_i} > x} \diff x \\
        & \leq \frac{\sigma_Y}{\sqrt{M}} \int_0^{\frac{a\sqrt{M}}{\sigma_Y}} \prob{\abs{\mathcal{N}(0, 1)} > x} + \frac{C_{a}}{\sqrt{M}} \diff x \tag{Berry-Esseen} \\
        & \leq \frac{\sigma_Y}{\sqrt{M}} \left(\frac{a C_{a}}{\sigma_Y} + \int_0^{\infty} \prob{\abs{\mathcal{N}(0, 1)} > x} \diff x \right) \\
        & = \frac{\sigma_Y}{\sqrt{M}} \left(\frac{a C_{a}}{\sigma_Y} + \expect{\abs{\mathcal{N}(0, 1)}} \right) \\
        & \leq \frac{C_a}{\sqrt{M}} + \frac{a}{\sqrt{M}} \expect{\abs{\mathcal{N}(0, 1)}} \\
        & = \mathcal{O}(M^{-1/2}),
    \end{align*}
    where the constant $C_{a}$ only depends on $a$ (which controls both the second and the third moment).

    \item \itempara{Proof of result~\ref{thm:asym_inf_negatives:itm:align}: The first term is minimized iff $f$ is perfectly aligned.}

    Note that for $u, v \in \sphere^d$, \begin{equation*}
        \norm{u - v}_2^2 = 2 - 2 \cdot u^T v.
    \end{equation*}

    Then the result follows directly the definition of perfect alignment, and the existence of perfectly aligned encoders (\eg, an encoder that maps every input to the same output vector).

    \item \itempara{Proof of result~\ref{thm:asym_inf_negatives:itm:uniform}: If perfectly uniform encoders exist, they form the exact minimizers of the second term.}

    For simplicity, we define the following notation: \begin{definition}
        $\forall \mu \in \mathcal{M}(\sphere^{d})$, $u \in \sphere^d$, we define the continuous and Borel measurable function  \begin{equation}
            U_\mu(u) \trieq \int_{\sphere^d} e^{u\T v / \tau} \diff \mu(v).
        \end{equation} with its range bounded in $[e^{-1/\tau}, e^{1/\tau}]$.
    \end{definition}

    Then the second term can be equivalently written as \begin{equation*}
        \expectunder[x \sim \pdata]{\log\expectunder[x^- \sim \pdata]{e^{f(x^-)\T f(x) / \tau}}} = \expectunder[x \sim \pdata]{\log U_{\pdata \circ f^{-1}} (f(x))},
    \end{equation*} where $\pdata \circ f^{-1} \in \mathcal{M}(\sphere^d)$ is the probability measure of features, \ie, the pushforward measure of $\pdata$ via $f$.

    We now consider the following relaxed problem, where the minimization is taken over $\mathcal{M}(\sphere^{d})$, all possible Borel probability measures on the hypersphere $\sphere^d$: %
    \begin{equation}
        \min_{\mu \in \mathcal{M}(\sphere^{d})} \int_{\sphere^d} \log U_\mu(u) \diff \mu(u).
        \label{supp:eq:thm:inf-neg-unif-relaxed-measure}
    \end{equation}

    Our strategy is to show that the unique minimizer of \Cref{supp:eq:thm:inf-neg-unif-relaxed-measure} is $\sigma_d$, from which the result~\ref{thm:asym_inf_negatives:itm:uniform} directly follows. The rest of the proof is structured in three parts.

    \begin{enumerate}
        \item
        \itempara{We show that minimizers of \Cref{supp:eq:thm:inf-neg-unif-relaxed-measure} exist, \ie, the above infimum is attained for some $\mu \in \mathcal{M}(\sphere^d)$.}\par\par

        Let $\{\mu_m\}_{m=1}^\infty$ be a sequence in $\mathcal{M}(\sphere^d)$ such that the infimum of \Cref{supp:eq:thm:inf-neg-unif-relaxed-measure} is reached in the limit: \begin{equation*}
            \lim_{m\rightarrow \infty} \int_{\sphere^d} \log U_{\mu_m}(u) \diff {\mu_m}(u) = \inf_{\mu \in \mathcal{M}(\sphere^{d})} \int_{\sphere^d}  \log U_\mu(u) \diff \mu(u).
        \end{equation*}
        From the Helly's Selection Theorem, let $\mu^*$ denote some weak$^*$ cluster point of this sequence. Then $\mu_m$ converges weak$^*$ to $\mu^*$ along a subsequence $m \in \mathcal{N} \in \N$. For simplicity and with a slight abuse of notation, we denote this convergent (sub)sequence of measures by $\{\mu_n\}_{n=1}^{\infty}$.

        We want to show that $\mu^*$ attains the limit (and thus the infimum), \ie,  \begin{equation}
            \int_{\sphere^d} \log U_{\mu^*}(u) \diff \mu^*(u) = \lim_{n\rightarrow \infty} \int_{\sphere^d} \log U_{\mu_n}(u) \diff {\mu_n}(u). \label{supp:eq:thm:converge-new-obj}
        \end{equation}

        In view of \Cref{supp:lemma:int-conv-fn-conv-measure}, since $\sphere^d$ is a compact second countable Hausdorff space and $\{\log U_{\mu_n}\}_n$ is uniformly bounded over $\sphere^d$, it remains to prove that $\{\log U_{\mu_n}\}_n$ is continuously convergent to $\log U_{\mu^*}$.

        Consider any convergent sequence of points $\{x_n\}_{n=1}^\infty \in \R^{d+1}$ \st $x_n \rightarrow x$ where $x \in \sphere^d$.

        Let $\delta_n = x_n - x$. By simply expanding $U_{\mu_n}$ and $\mu_{\mu^*}$, we have
        \begin{equation*}
            e^{-\norm{\delta_n} / \tau} U_{\mu_n}(x) \leq U_{\mu_n}(x_n) \leq e^{\norm{\delta_n} / \tau} U_{\mu_n}(x).
        \end{equation*}
        Since both the upper and the lower bound converge to $U_{\mu^*}(x)$  (by the weak $^*$ convergence of $\{\mu_n\}_n$ to $\mu^*$), $U_{\mu_n}(x_n)$ must as well. We have proved the continuous convergence of $\{\log U_{\mu_n}\}_n$ to $\log U_{\mu^*}$.

        Therefore, the limit in \Cref{supp:eq:thm:converge-new-obj} holds. The infimum is thus attained at $\mu^*$: \begin{equation*}
            \lim_{n \rightarrow \infty} \int_u \log U_{\mu_n}(u) \diff {\mu_n} = \int_u \log U_{\mu^*}(u) \diff \mu^*.
        \end{equation*}

        \item
        \itempara{We show that $U_{\mu^*}$ is constant $\mu^*$-almost surely for any minimizer $\mu^*$ of \Cref{supp:eq:thm:inf-neg-unif-relaxed-measure}.}

        Let $\mu^*$ be any solution of \Cref{supp:eq:thm:inf-neg-unif-relaxed-measure}: \begin{equation*}
            \mu^* \in \argmin_{\mu \in \mathcal{M}(\sphere^{d})} \int_u  \log U_\mu(u) \diff \mu.
        \end{equation*}

        Consider the Borel sets where $\mu^*$ has positive measure: $\mathcal{T} \trieq \{T \in \mathcal{B}(\sphere^d) \colon \mu^*(T) > 0\}$. For any $T \in \mathcal{T}$, let $\mu^*_T$ denote the conditional distribution of $\mu^*$ on $T$, \ie, $\forall A \in \mathcal{B}(\sphere^d)$, \begin{equation*}
            \mu^*_T(A) = \frac{\mu^*(A \cap T)}{\mu^*(T)}.
        \end{equation*}

        Note that for any such $T \in \mathcal{T}$, the mixture $(1 - \alpha)\mu^* + \alpha \mu^*_T$ is a valid probability distribution (\ie, in $\mathcal{M}(\sphere^d)$) for $\alpha \in (-\mu^*(T), 1)$, an open interval containing $0$.

        By the first variation, we must have%
         \begin{align}
            0
            & = \frac{\partial}{\partial \alpha}
            \int_{\sphere^d} \log U_{(1-\alpha) \mu^* + \alpha \mu^*_T}(u) \diff ((1-\alpha) \mu^* + \alpha \mu^*_T)(u)
            \evalat{\alpha=0} \notag\\
            & =
            \frac{\partial}{\partial \alpha} (1-\alpha)
            \int_{\sphere^d} \log U_{(1-\alpha) \mu^* + \alpha \mu^*_T}(u) \diff \mu^*(u)
            \evalat{\alpha=0}
            + \frac{\partial}{\partial \alpha} \alpha
            \int_{\sphere^d} \log U_{(1-\alpha) \mu^* + \alpha \mu^*_T}(u) \diff \mu^*_T(u)
            \evalat{\alpha=0} \notag\\
            & =
            - \int_{\sphere^d} \log U_{(1-\alpha) \mu^* + \alpha \mu^*_T}(u) \diff \mu^*(u)
            \evalat{\alpha=0}
            +
            \frac{\partial}{\partial \alpha}
            \int_{\sphere^d} \log U_{(1-\alpha) \mu^* + \alpha \mu^*_T}(u) \diff \mu^*(u)
            \evalat{\alpha=0} \notag\\
            & \qquad\qquad +
            \int_{\sphere^d} \log U_{(1-\alpha) \mu^* + \alpha \mu^*_T}(u) \diff \mu^*_T(u)
            \evalat{\alpha=0}
            +
            0 \cdot \frac{\partial}{\partial \alpha}
            \int_{\sphere^d} \log U_{(1-\alpha) \mu^* + \alpha \mu^*_T}(u) \diff \mu^*_T(u)
            \evalat{\alpha=0} \notag\\
            & =
            - \int_{\sphere^d} \log U_{\mu^*}(u) \diff \mu^*(u)
            +
            \int_{\sphere^d} \frac{U_{\mu^*_T}(u) - U_{\mu^*}(u)}{U_{\mu^* }(u)} \diff \mu^*(u) \notag\\
            & \qquad\qquad +
            \int_{\sphere^d} \log U_{\mu^*}(u) \diff \mu^*_T(u)
            +
            0 \cdot
            \int_{\sphere^d} \frac{U_{\mu^*_T}(u) - U_{\mu^*}(u)}{U_{\mu^* }(u)} \diff \mu^*_T(u) \notag\\
            & =
            \int_{\sphere^d} \frac{U_{\mu^*_T}(u)}{U_{\mu^* }(u)} \diff \mu^*(u)
            +
            \int_{\sphere^d} \log U_{\mu^*}(u) \diff (\mu^*_T - \mu^*)(u) - 1 \label{supp:thm:asym-inf-neg:eq:alpha-first-variation},
        \end{align}where the Leibniz rule along with the boundedness of $U_{\mu^*}$ and $U_{\mu^*_{T_n}}$ together justify the exchanges of integration and differentiation.

        Let $\{T_n\}_{n=1}^\infty$ be a sequence of sets in $\mathcal{T}$ such that \begin{equation*}
            \lim_{n \rightarrow \infty} \int_{\sphere^d} U_{\mu^*}(u) \diff \mu_{T_n}^*(u) = \sup_{T \in \mathcal{T}} \int_{\sphere^d} U_{\mu^*}(u) \diff \mu_{T}^*(u) \trieq U^*,
        \end{equation*} where the supremum must exist since $U_{\mu^*}$ is bounded above.

        Because $U_{\mu^*}$ is a continuous and Borel measurable function, we have $\{u \colon U_{\mu^*}(u) > U^*\} \in \mathcal{B}(\sphere^d)$ and thus \begin{align*}
            \mu^*(\{u \colon U_{\mu^*}(u) > U^*\}) & = 0, \\
            \mu^*_{T_n}(\{u \colon U_{\mu^*}(u) > U^*\}) & = 0, & \forall n = 1, 2, \dots,
        \end{align*}
        otherwise $\{u \colon U_{\mu^*}(u) > U^*\} \in \mathcal{T}$, contradicting the definition of $U^*$ as the supremum.

        Asymptotically, $U_{\mu^*}$ is constant $\mu^*_{T_n}$-almost surely: \begin{align*}
            & \int_{\sphere^d} \abs{ U_{\mu^*}(u) - \int_{\sphere^d} U_{\mu^*}(u') \diff \mu_{T_n}^*(u')} \diff \mu_{T_n}^*(u) \\
            & \qquad\qquad = 2 \int_{\sphere^d} \max \left(0,~U_{\mu^*}(u) - \int_{\sphere^d} U_{\mu^*}(u') \diff \mu_{T_n}^*(u') \right) \diff \mu_{T_n}^*(u) \\
            & \qquad\qquad \leq 2 (U^* - \int_{\sphere^d} U_{\mu^*}(u) \diff \mu_{T_n}^*(u)) \\
            & \qquad\qquad \rightarrow 0, & \text{as $n \rightarrow \infty$,}
        \end{align*}
        where the inequality follows the boundedness of $U_{\mu^*}$ and that $\mu^*_{T_n}(\{u \colon U_{\mu^*}(u) > U^*\}) = 0$.

        Therefore, given the continuity of $\log$ and the boundedness of $U_{\mu^*}$, we have \begin{equation*}
            \lim_{n\rightarrow \infty} \int_{\sphere^d} \log U_{\mu^*}(u) \diff \mu_{T_n}^*  = \log U^*.
        \end{equation*}


        \Cref{supp:thm:asym-inf-neg:eq:alpha-first-variation} gives that $\forall n = 1, 2, \dots$, \begin{align*}
            1
            & = \int_{\sphere^d} \frac{U_{\mu_{T_n}^*}(u)}{U_{\mu^* }(u)} \diff \mu^*
            +
            \int_{\sphere^d} \log U_{\mu^*}(u) \diff (\mu_{T_n}^* - \mu^*) \notag\\
            & \geq \frac{1}{U^*} \int_{\sphere^d} U_{\mu_{T_n}^*}(u) \diff \mu^*(u)
            +
            \int_{\sphere^d} \log U_{\mu^*}(u) \diff \mu_{T_n}^*
            -
            \int_{\sphere^d} \log U_{\mu^*}(u) \diff \mu^* \\
            & = \frac{1}{U^*} \int_{\sphere^d} U_{\mu^*}(u) \diff \mu_{T_n}^*(u)
            +
            \int_{\sphere^d} \log U_{\mu^*}(u) \diff \mu_{T_n}^*
            -
            \int_{\sphere^d} \log U_{\mu^*}(u) \diff \mu^*,
        \end{align*}
        where the inequality follows the boundedness of $\frac{U_{\mu_{T_n}^*}}{U_{\mu^*}}$ and that $\mu^*(\{u \colon U_{\mu^*}(u) > U^*\}) = 0$.

        Taking the limit of $n \rightarrow \infty$ on both sides, we have \begin{align*}
            1 = \lim_{n \rightarrow \infty} 1
            & \geq
            \frac{1} {U^*} \lim_{n \rightarrow \infty} \int_{\sphere^d} U_{\mu^*}(u) \diff \mu_{T_n}^*(u)
            +
            \lim_{n\rightarrow \infty} \int_{\sphere^d} \log U_{\mu^*}(u) \diff \mu_{T_n}^*(u)
            -
            \int_{\sphere^d} \log U_{\mu^*}(u) \diff \mu^*(u) \\
            & = 1 + \log U^* - \int_{\sphere^d} \log U_{\mu^*}(u) \diff \mu^*(u) \\
            & \geq 1 + \log U^* - \log \int_{\sphere^d} U_{\mu^*}(u) \diff \mu^*(u) \\
            & \geq 1,
        \end{align*}where the last inequality holds because the supremum taken over $\mathcal{T} \supset \{\sphere^d\}$.

        Since $1=1$, all inequalities must be equalities. In particular, \begin{equation*}
            \int_{\sphere^d} \log U_{\mu^*}(u) \diff \mu^*(u) =  \log \int_{\sphere^d} U_{\mu^*}(u) \diff \mu^*(u).
        \end{equation*}
        That is, for any solution $\mu^*$ of \Cref{supp:eq:thm:inf-neg-unif-relaxed-measure}, $U_{\mu^*}$ must be constant $\mu^*$-almost surely.

        \item \itempara{We show that $\sigma_d$ is the unique minimizer of the relaxed problem in \Cref{supp:eq:thm:inf-neg-unif-relaxed-measure}.}

        Let $S \subset \mathcal{M}(\sphere^d)$ be the set of measures where the above property holds: \begin{equation*}
            S \trieq \left\{\mu \in \mathcal{M}(\sphere^{d}) \colon U_{\mu} \text{ is constant $\mu$-almost surely} \right\}.
        \end{equation*}

        The problem in \Cref{supp:eq:thm:inf-neg-unif-relaxed-measure} is thus equivalent to minimizing over $S$: \begin{align*}
            \argmin_{\mu \in \mathcal{M}(\sphere^{d})} \int_{\sphere^d}  \log U_\mu(u) \diff \mu(u)
            & = \argmin_{\mu \in S} \int_{\sphere^d}  \log U_\mu(u) \diff \mu(u) \\
            & = \argmin_{\mu \in S} \log \int_{\sphere^d} U_\mu(u) \diff \mu(u) \\
            & = \argmin_{\mu \in S} \log \int_{\sphere^d} \int_{\sphere^d} e^{u\T v / \tau} \diff \mu(v) \diff \mu(u) \\
            & = \argmin_{\mu \in S} \left( \frac{1}{\tau} + \log  \int_{\sphere^d} \int_{\sphere^d} e^{-\frac{1}{2 \tau} \norm{u - v}^2} \diff \mu(v) \diff \mu(u) \right) \\
            & = \argmin_{\mu \in S} \int_{\sphere^d} \int_{\sphere^d} G_{\frac{1}{2\tau}}(u, v) \diff \mu(v) \diff \mu(u).
        \end{align*}

        By \Cref{prop:continuous-problem} and $\tau > 0$, we know that the uniform distribution $\sigma_d$ is the unique solution to \begin{equation}
            \argmin_{\mu \in \mathcal{M}(\sphere^d)} \int_{\sphere^d} \int_{\sphere^d} G_{\frac{1}{2\tau}}(u, v) \diff \mu(v) \diff \mu(u). \label{supp:thm:asymp-inf-neg:eq:unif-equiv-gaussian-prob}
        \end{equation}

        Since $\sigma_d \in S$, it must also be the unique solution to \Cref{supp:eq:thm:inf-neg-unif-relaxed-measure}.
    \end{enumerate}

    Finally, if perfectly uniform encoders exist, $\sigma_d$ is realizable, and they are the exact encoders that realize it. Hence, in such cases, they are the exact minimizers of \begin{equation*}
        \min_f \expectunder[x \sim \pdata]{\log\expectunder[x^- \sim \pdata]{e^{f(x^-)\T f(x) / \tau}}}.
    \end{equation*}
\end{itemize}
\end{proof}

\paragraph{Relation between \Cref{thm:asym_inf_negatives}, $\lalign$ and $\lunif$.} The first term of \Cref{eq:contrastive_loss_limit} is equivalent with $\lalign$ when $\alpha = 2$, up to a constant and a scaling. In the above proof, we showed that the second term favors uniformity, via the feature distribution that minimizes the pairwise Gaussian kernel (see \Cref{supp:thm:asymp-inf-neg:eq:unif-equiv-gaussian-prob}): \begin{equation}
    \argmin_{\mu \in \mathcal{M}(\sphere^d)} \int_{\sphere^d} \int_{\sphere^d} G_{\frac{1}{2\tau}}(u, v) \diff \mu(v) \diff \mu(u), \label{supp:eq:asym-inf-neg-unif-term-measure-relax}
\end{equation} which can be alternatively viewed as the relaxed problem of optimizing for the uniformity loss $\lunif$: \begin{equation}
    \argmin_f \lunif(f; \frac{1}{2\tau}) = \argmin_f \expect[x, y \iidsim \distndata] {G_{\frac{1}{2\tau}}(f(x), f(y))}. \label{supp:eq:asym-inf-neg-unif-term-realizable-nonrelax}
\end{equation} The relaxation comes from the observation that \Cref{supp:eq:asym-inf-neg-unif-term-measure-relax} minimizes over all feature distributions on $\sphere^d$, while \Cref{supp:eq:asym-inf-neg-unif-term-realizable-nonrelax} only considers the realizable ones.

\paragraph{Relation between \Cref{supp:eq:thm:inf-neg-unif-relaxed-measure} and minimizing average pairwise Gaussian potential (\ie, minimizing $\lunif$).} In view of the \Cref{prop:continuous-problem} and the proof of \Cref{thm:asym_inf_negatives}, we know that the uniform distribution $\sigma_d$ is the unique minimizer of both of the following problems: \begin{align*}
    \{\sigma_d\} & = \min_{\mu \in \mathcal{M}(\sphere^d)}  \log \int_{\sphere^d} \int_{\sphere^d} e^{u\T v /\tau} \diff \mu(v) \diff \mu(u), \\
    \{\sigma_d\} & = \min_{\mu \in \mathcal{M}(\sphere^d)} \int_{\sphere^d} \log  \int_{\sphere^d} e^{u\T v /\tau} \diff \mu(v) \diff \mu(u).
\end{align*}
So pushing the $\log$ inside the outer integral doesn't change the solution. However, if we push the $\log$ all the way inside the inner integral, the problem becomes equivalent with minimizing the norm of the mean, \ie, \begin{equation*}
    \min_{\mu \in \mathcal{M}(\sphere^d)} \expect[U \sim \mu]{U}\T\expect[U \sim \mu]{U},
\end{equation*}
which is minimized for any distribution with mean being the all-zeros vector $0$, \eg, $\frac{1}{2} \delta_{u} + \frac{1}{2} \delta_{-u}$ for any $u \in \sphere^d$ (where $\delta_u$ is the Dirac delta distribution at $u$ \st $\delta_u(S) = \mathbbm{1}_{S}(u)$, $\forall S \in \mathcal{B}(\sphere^d)$). Therefore, the location of the $\log$ is important.

\begin{theorem}[Single negative sample] \label{thm:asym_single_negative}
    If perfectly aligned and uniform encoders exist, they form the exact minimizers of the contrastive loss $\lcontr(f; \tau, M)$ for fixed $\tau > 0$ and $M = 1$.
\end{theorem}
\begin{proof}[Proof of \Cref{thm:asym_single_negative}]
    Since $M=1$, we have \begin{align}
        \lcontr(f; \tau, 1)
        & = \expectunder[\substack{
            (x, y) \sim \ppos \\
            x^- \sim \pdata
            }]{-\frac{1}{\tau} f(x)\T f(y) + \log \left( e^{f(x)\T f(y) / \tau}+  e^{f(x^-)\T f(x) / \tau} \right)} \notag \\
        & \geq \expectunder[\substack{
            x \sim \pdata \\
            x^- \sim \pdata
            }]{-\frac{1}{\tau} + \log \left( e^{1 / \tau}+  e^{f(x^-)\T f(x) / \tau} \right)}\label{supp:thm:single-neg:eq:align-relax} \\
        & \geq -\frac{1}{\tau} + \min_{\mu \in \mathcal{M}(\sphere^d)} \int_{\sphere^d} \int_{\sphere^d} \log \left( e^{1 / \tau}+  e^{u\T v / \tau} \right) \diff \mu(u) \diff \mu(v) \label{supp:thm:single-neg:eq:unif-relax} \\
        & = -\frac{1}{\tau} + \min_{\mu \in \mathcal{M}(\sphere^d)} \int_{\sphere^d} \int_{\sphere^d} \log \left( e^{1 / \tau}+  e^{(2 - \norm{u - v}_2^2) / (2 \tau)} \right) \diff \mu(u) \diff \mu(v). \notag
    \end{align}
    By the definition of perfect alignment, the equality in \Cref{supp:thm:single-neg:eq:align-relax} is satisfied iff $f$ is perfectly aligned.

    Consider the function $f \colon (0, 4] \rightarrow \R_+$ defined as \begin{equation*}
        f(t) = \log(e^{\frac{1}{\tau}} + e^{\frac{2 - t}{2\tau}}).
    \end{equation*} It has the following properties: \begin{itemize}
        \item $-f'(t) = \frac{1}{2\tau} \frac{ e^{-\frac{t}{2\tau}} }{1 + e^{-\frac{t}{2\tau}}} = \frac{1}{2\tau} (1 - (1 + e^{-\frac{t}{2\tau}})^{-1})$ is strictly completely monotone on $(0, +\infty)$:

        $\forall t \in (0,  +\infty)$, \begin{align*}
            \frac{1}{2\tau} (1 - (1 + e^{-\frac{t}{2\tau}})^{-1}) & > 0 \\
            (-1)^n \frac{\diff^n }{\diff t^n} \frac{1}{2\tau} (1 - (1 + e^{-\frac{t}{2\tau}})^{-1}) & = \frac{n!}{(2\tau)^{n+1}} (1 + e^{-\frac{t}{2\tau}})^{-(n+1)} > 0, \qquad\qquad n = 1, 2, \dots. \\
        \end{align*}
        \item $f$ is bounded on $(0, 4]$.
    \end{itemize}
    In view of \Cref{supp:lemma:strict-pd-uniform}, we have that the equality in \Cref{supp:thm:single-neg:eq:unif-relax} is satisfied iff the feature distribution induced by $f$ (\ie, the pushforward measure $\pdata \circ f^{-1}$) is $\sigma_d$, that is, in other words, $f$ is perfectly uniform.

    Therefore, \begin{equation*}
        \lcontr(f; \tau, 1) \geq -\frac{1}{\tau} + \int_{\sphere^d} \int_{\sphere^d} \log \left( e^{1 / \tau}+  e^{u\T v / \tau} \right) \diff \sigma_d(u) \diff \sigma_d(v) = \text{constant independent of $f$},
    \end{equation*}where equality is satisfied iff $f$ is perfectly aligned and uniform. This concludes the proof.
\end{proof}

\paragraph{Difference between conditions of Theorems~\ref{thm:asym_inf_negatives}~and~\ref{thm:asym_single_negative}.}We remark that the statement in \Cref{thm:asym_single_negative} is weaker than the previous \Cref{thm:asym_inf_negatives}. \Cref{thm:asym_single_negative} is conditioned on the existence perfectly aligned and uniform encoders. It only shows that $\lcontr(f; \tau, M=1)$  favors alignment under the condition that perfect uniformity is realizable, and vice versa. In \Cref{thm:asym_inf_negatives}, $\lcontr$ decomposes into two terms, each favoring alignment and uniformity. Therefore, the decomposition in \Cref{thm:asym_inf_negatives} is exempof t from this constraint.

\section{Experiment Details}\label{supp:sec:expr-details}
All experiments are performed on 1-4 NVIDIA Titan Xp, Titan X PASCAL, Titan RTX, or 2080 Ti GPUs.

\subsection{\cifar, \stl and \nyudepth Experiments}
For \cifar, \stl and \nyudepth experiments, we use the following settings, unless otherwise stated in Tables~\ref{supp:tbl:stl10-big}~and~\ref{supp:tbl:nyudepth-big} below: \begin{itemize}
    \item Standard data augmentation procedures are used for generating positive pairs, including resizing, cropping, horizontal flipping, color jittering, and random grayscale conversion. This follows prior empirical work in contrastive representation learning \citep{wu2018unsupervised,tian2019contrastive,hjelm2018learning,bachman2019learning}.
    \item Neural network architectures follow the corresponding experiments on these datasets in \citet{tian2019contrastive}. For \nyudepth evaluation, the architecture of the depth prediction CNN is described in \Cref{supp:tbl:nyudepth-cnn-depth-predictor-arch}.
    \item We use minibatch stochastic gradient descent (SGD) with $0.9$ momentum and $0.0001$ weight decay.
    \item We use linearly scaled learning rate ($0.12$ per $256$ batch size) \citep{goyal2017accurate}. \begin{itemize}
        \item \cifar and \stl: Optimization is done over $200$ epochs, with learning rate decayed by a factor of $0.1$ at epochs $155$, $170$, and $185$.
        \item \nyudepth: Optimization is done over $400$ epochs, with learning rate decayed by a factor of $0.1$ at epochs $310$, $340$, and $370$.
    \end{itemize}
    \item Encoders are optimized over the training split. For evaluation, we freeze the encoder, and train classifiers / depth predictors on the training set samples, and test on the validation split. \begin{itemize}
        \item \cifar and \stl: We use standard train-val split. Linear classifiers are trained with Adam \citep{kingma2014adam} over $100$ epochs, with $\beta_1=0.5, \beta_2 = 0.999, \epsilon=10^{-8}$, $128$ batch size, and an initial learning rate of $0.001$, decayed by a factor of $0.2$ at epochs $60$ and $80$.
        \item \nyudepth: We use the train-val split on the $1449$ labeled images from \citet{Silberman:ECCV12}. Depth predictors are trained with Adam \citep{kingma2014adam} over $120$ epochs, with $\beta_1=0.5, \beta_2 = 0.999, \epsilon=10^{-8}$, $128$ batch size, and an initial learning rate of $0.003$, decayed by a factor of $0.2$ at epochs $70$, $90$, $100$, and $110$.
    \end{itemize}
\end{itemize}

\begin{table*}[t!]%

\newcommand{\mround}[1]{\round{#1}{4}}
\renewcommand{\arraystretch}{1.2}
\newcommand{\NA}{---}
\centering
\small
\begin{tabular}{|c|c|c|c|c|c|c|c|}
    \hline
    \multirow{2}{*}{Operator} &
    \multirow{2}{*}{\shortstack{Input\\Spatial Shape}} &
    \multirow{2}{*}{\shortstack{Input\\\#Channel}} &
    \multirow{2}{*}{\shortstack{Kernel\\Size}} &
    \multirow{2}{*}{Stride} &
    \multirow{2}{*}{Padding} &
    \multirow{2}{*}{\shortstack{Output\\Spatial Shape}} &
    \multirow{2}{*}{\shortstack{Output\\\#Channel}} \\

    & & & & & & & \\
    \hline\hline

    Input &
    $[h_\mathsf{in}, w_\mathsf{in}]$ &
    $c_\mathsf{in}$ &
    \NA &
    \NA &
    \NA &
    $[h_\mathsf{in}, w_\mathsf{in}]$ &
    $c_\mathsf{in}$ \\
    \hline

    Conv.~Transpose + BN + ReLU &
    $[h_\mathsf{in}, w_\mathsf{in}]$ &
    $c_\mathsf{in}$ &
    3 &
    2 &
    1 &
    $[2h_\mathsf{in}, 2w_\mathsf{in}]$ &
    $\floor{c_\mathsf{in} / 2}$ \\
    \hline

    Conv.~Transpose + BN + ReLU &
    $[2h_\mathsf{in}, 2w_\mathsf{in}]$ &
    $\floor{c_\mathsf{in} / 2}$ &
    3 &
    2 &
    1 &
    $[4h_\mathsf{in}, 4w_\mathsf{in}]$ &
    $\floor{c_\mathsf{in} / 4}$ \\
    \hline

    $\vdots$ &
    $\vdots$ &
    $\vdots$ &
    $\vdots$ &
    $\vdots$ &
    $\vdots$ &
    $\vdots$ &
    $\vdots$ \\
    \hline

    Conv.~Transpose + BN + ReLU &
    $[h_\mathsf{out} / 2, w_\mathsf{out} / 2]$ &
    $\floor{c_\mathsf{in} / 2^{n-1}}$ &
    3 &
    2 &
    1 &
    $[h_\mathsf{out}, w_\mathsf{out}]$ &
    $\floor{c_\mathsf{in} / 2^n}$ \\
    \hline

    Conv. &
    $[h_\mathsf{out}, w_\mathsf{out}]$ &
    $\floor{c_\mathsf{in} / 2^n}$ &
    3 &
    1 &
    1 &
    $[h_\mathsf{out}, w_\mathsf{out}]$ &
    $1$ \\
    \hline


\end{tabular}
\caption{\nyudepth CNN depth predictor architecture. Each \textrm{Conv.\hspace{1.5pt}Transpose+BN+ReLU} block increases the spatial shape by a factor of $2$, where BN denotes Batch Normalization \citep{ioffe2015batch}. A sequence of such blocks computes a tensor of the correct spatial shape, from an input containing intermediate activations of a CNN encoder (which downsamples the input RGB image by a power of $2$). A final convolution at the end computes the single-channel depth prediction. } \label{supp:tbl:nyudepth-cnn-depth-predictor-arch}

\end{table*}

At each SGD iteration, a minibatch of $K$ positive pairs is sampled $\{(x_i, y_i)\}_{i=1}^K$, and the three losses for this minibatch are calculated as following: \begin{itemize}
    \item $\lcontr$: For each $x_i$, the sample contrastive loss is taken with the positive being $y_i$, and the negatives being $\{y_j\}_{j \neq i}$. For each $y_i$, the sample loss is computed similarly. The minibatch loss is calculated by aggregating these $2K$ terms: \begin{equation*}
        \frac{1}{2K} \sum_{i = 1}^{K}
        \log \frac{e^{f(x_i)\T f(y_i) / \tau}}{\sum_{j = 1}^{K} e^{f(x_i)\T f(y_j) / \tau}} +
        \frac{1}{2K} \sum_{i = 1}^{K} \log \frac{e^{f(x_i)\T f(y_i) / \tau}}{\sum_{j = 1}^{K} e^{f(x_j)\T f(y_i) / \tau}}.
    \end{equation*} This calculation follows empirical practices and is similar to \citet{oord2018representation,henaff2019data}, and \textit{end-to-end} in \citet{he2019momentum}.
    \item $\lalign$: The minibatch alignment loss is straightforwardly computed as \begin{equation*}
        \frac{1}{K} \sum_{i=1}^{K} \norm{f(x_i) - f(y_i)}_2^\alpha.
    \end{equation*}
    \item $\lunif$: The minibatch uniform loss is calculated by considering each pair of $\{x_i\}_i$ and $\{y_i\}_i$: \begin{equation*}
        \frac{1}{2} \log \bigg( \frac{2}{K (K - 1)} \sum_{i \neq j} e^{-t \norm{f(x_i) - f(x_j)}_2^2} \bigg) + \frac{1}{2} \log \bigg( \frac{2}{K (K - 1)} \sum_{i \neq j} e^{-t \norm{f(y_i) - f(y_j)}_2^2} \bigg).
    \end{equation*}
\end{itemize}

Tables~\ref{supp:tbl:stl10-big}~and~\ref{supp:tbl:nyudepth-big} below describe the full specifications of all $304$ \stl and $64$ \nyudepth encoders. These experiment results are visualized in main paper \Cref{fig:expr_scatter}, showing a clear connection between representation quality and $\lalign$ \& $\lunif$ metrics.

\begin{table*}[t!]%

\resizebox{
  \ifdim\width>\textwidth
    \textwidth
  \else
    \width
  \fi
}{!}{%
\centering
\small
\renewcommand{\arraystretch}{1.2}
\begin{tabular}{|c|c|c|c|c|c|c|c|c|c|}
    \hline
    \multicolumn{10}{|c|}{\multirow{2}{*}{\normalsize\imagenetsubset Classes}} \\
    \multicolumn{10}{|c|}{} \\\hline\hline
    \texttt{n02869837} & \texttt{n01749939} & \texttt{n02488291} & \texttt{n02107142} & \texttt{n13037406} & \texttt{n02091831} & \texttt{n04517823} & \texttt{n04589890} & \texttt{n03062245} & \texttt{n01773797} \\ \hline
    \texttt{n01735189} & \texttt{n07831146} & \texttt{n07753275} & \texttt{n03085013} & \texttt{n04485082} & \texttt{n02105505} & \texttt{n01983481} & \texttt{n02788148} & \texttt{n03530642} & \texttt{n04435653} \\ \hline
    \texttt{n02086910} & \texttt{n02859443} & \texttt{n13040303} & \texttt{n03594734} & \texttt{n02085620} & \texttt{n02099849} & \texttt{n01558993} & \texttt{n04493381} & \texttt{n02109047} & \texttt{n04111531} \\ \hline
    \texttt{n02877765} & \texttt{n04429376} & \texttt{n02009229} & \texttt{n01978455} & \texttt{n02106550} & \texttt{n01820546} & \texttt{n01692333} & \texttt{n07714571} & \texttt{n02974003} & \texttt{n02114855} \\ \hline
    \texttt{n03785016} & \texttt{n03764736} & \texttt{n03775546} & \texttt{n02087046} & \texttt{n07836838} & \texttt{n04099969} & \texttt{n04592741} & \texttt{n03891251} & \texttt{n02701002} & \texttt{n03379051} \\ \hline
    \texttt{n02259212} & \texttt{n07715103} & \texttt{n03947888} & \texttt{n04026417} & \texttt{n02326432} & \texttt{n03637318} & \texttt{n01980166} & \texttt{n02113799} & \texttt{n02086240} & \texttt{n03903868} \\ \hline
    \texttt{n02483362} & \texttt{n04127249} & \texttt{n02089973} & \texttt{n03017168} & \texttt{n02093428} & \texttt{n02804414} & \texttt{n02396427} & \texttt{n04418357} & \texttt{n02172182} & \texttt{n01729322} \\ \hline
    \texttt{n02113978} & \texttt{n03787032} & \texttt{n02089867} & \texttt{n02119022} & \texttt{n03777754} & \texttt{n04238763} & \texttt{n02231487} & \texttt{n03032252} & \texttt{n02138441} & \texttt{n02104029} \\ \hline
    \texttt{n03837869} & \texttt{n03494278} & \texttt{n04136333} & \texttt{n03794056} & \texttt{n03492542} & \texttt{n02018207} & \texttt{n04067472} & \texttt{n03930630} & \texttt{n03584829} & \texttt{n02123045} \\ \hline
    \texttt{n04229816} & \texttt{n02100583} & \texttt{n03642806} & \texttt{n04336792} & \texttt{n03259280} & \texttt{n02116738} & \texttt{n02108089} & \texttt{n03424325} & \texttt{n01855672} & \texttt{n02090622} \\ \hline
\end{tabular}%
}
\caption{$100$ randomly selected \imagenet classes forming the \imagenetsubset subset. These classes are the same as the ones used by \citet{tian2019contrastive}.} \label{supp:tbl:imagenet100-classes}

\end{table*}

\subsection{\imagenet and \imagenetsubset with Momentum Contrast (MoCo) Variants}

\paragraph{MoCo and MoCo v2 with $\lalign$ and $\lunif$. } At each SGD iteration, let \begin{itemize}
    \item $K$ be the minibatch size,
    \item $\{f(x_i)_i\}_{i=1}^K$ be the batched query features encoded by the current up-to-date encoder $f$ (\ie, $\mathtt{q}$ in Algorithm~1 of \citet{he2019momentum}),
    \item $\{f_\textsf{EMA}(y_i)\}_{i=1}^K$ be the batched key features encoded by the exponential moving average encoder $f_\textsf{EMA}$ (\ie, $\mathtt{k}$ in Algorithm~1 of \citet{he2019momentum}),
    \item $\{\mathtt{queue}_j\}_{j=1}^N$ be the feature queue, where $N$ is the queue size.
\end{itemize} $\lalign$ and $\lunif$ for this minibatch are calculated as following: \begin{itemize}
    \item $\lalign$: The minibatch alignment loss is computed as disparity between features from the two encoders: \begin{equation*}
        \frac{1}{K} \sum_{i=1}^{K} \norm{f(x_i) - f_\textsf{EMA}(y_i)}_2^\alpha.
    \end{equation*}
    \item $\lunif$: We experiment with two forms of $\lunif$: \begin{enumerate}
        \item Only computing pairwise distance between $\{f(x_i)\}_i$ and $\{\mathtt{queue}_j\}_j$:
        \begin{equation}
            \log \bigg( \frac{1}{NK} \sum_{i=1}^{K} \sum_{j=1}^{N} e^{-t \norm{f(x_i) - \mathtt{queue}_j}_2^2} \bigg). \label{supp:eq:moco:lunif-queue-only}
        \end{equation}
        \item Also computing pairwise distance inside $\{f(x_i)\}_i$:
        \begin{equation}
            \log \bigg( \frac{2}{2 NK + K (K-1)} \sum_{i=1}^{K} \sum_{j=1}^{N} e^{-t \norm{f(x_i) - \mathtt{queue}_j}_2^2} + \frac{2}{2 NK + K (K-1)} \sum_{i \neq j} e^{-t \norm{f(x_i) - f(x_j)}_2^2} \bigg). \label{supp:eq:moco:lunif-queue-and-intra-batch}
        \end{equation}
    \end{enumerate}
\end{itemize}

\subsubsection{\imagenetsubset with MoCo} \label{supp:sec:imagenet-100-moco}

\paragraph{\imagenetsubset details.} We use the same \imagenetsubset sampled by \citet{tian2019contrastive}, containing the $100$ randomly selected classes listed in \Cref{supp:tbl:imagenet100-classes}.%

\paragraph{MoCo settings.}
Our MoCo experiment settings below mostly follow \citet{he2019momentum} and the unofficial implementation by \citet{tian2019cmcgithub}, because the official implementation was not released at the time of performing these analyses: \begin{itemize}
    \item Standard data augmentation procedures are used for generating positive pairs, including resizing, cropping, horizontal flipping, color jittering, and random grayscale conversion, following \citet{tian2019cmcgithub}.
    \item Encoder architecture is ResNet50 \citep{he2016deep}.
    \item We use minibatch stochastic gradient descent (SGD) with $128$ batch size, $0.03$ initial learning rate, $0.9$ momentum and $0.0001$ weight decay.
    \item Optimization is done over $240$ epochs, with learning rate decayed by a factor of $0.1$ at epochs $120$, $160$, and $200$.
    \item We use $0.999$ exponential moving average factor, following \citet{he2019momentum}.
    \item For evaluation, we freeze the encoder, and train a linear classifier on the training set samples, and test on the validation split.  Linear classifiers are trained with minibatch SGD over $60$ epochs, with $256$ batch size, and an initial learning rate of $10$, decayed by a factor of $0.2$ at epochs $30$, $40$, and $50$.
\end{itemize}

\Cref{supp:tbl:imagenet100-big} below describes the full specifications of all $45$ \imagenetsubset encoders. These experiment results are visualized in main paper \Cref{fig:expr_scatter_imagenet100}, showing a clear connection between representation quality and $\lalign$ \& $\lunif$ metrics.

\subsubsection{\imagenet with MoCo v2}

\paragraph{MoCo v2 settings.}
Our MoCo v2 experiment settings directly follow \citet{chen2020improved} and the official implementation \citep{chen2020mocov2github}: \begin{itemize}
    \item Standard data augmentation procedures are used for generating positive pairs, including resizing, cropping, horizontal flipping, color jittering, random grayscale conversion, and random Gaussian blurring, following \citet{chen2020mocov2github}.
    \item Encoder architecture is ResNet50 \citep{he2016deep}.
    \item We use minibatch stochastic gradient descent (SGD) with $256$ batch size, $0.03$ initial learning rate, $0.9$ momentum and $0.0001$ weight decay.
    \item Optimization is done over $200$ epochs, with learning rate decayed by a factor of $0.1$ at epochs $120$ and $160$.
    \item We use $0.999$ exponential moving average factor, $65536$ queue size, $128$ feature dimensions.
    \item For evaluation, we freeze the encoder, and train a linear classifier on the training set samples, and test on the validation split.  Linear classifiers are trained with minibatch SGD over $100$ epochs, with $256$ batch size, and an initial learning rate of $30$, decayed by a factor of $0.1$ at epochs $60$ and $80$.
\end{itemize}

Unlike the MoCo experiments on \imagenetsubset, which were based on unofficial implementations for reasons stated in Sec.~\ref{supp:sec:imagenet-100-moco}, the MoCo v2 experiments on full \imagenet were based on the official implementation by \citet{chen2020mocov2github}. We provide a reference implementation that can fully reproduce the results in \Cref{tbl:expr_imagenet} at \href{https://github.com/SsnL/moco_align_uniform}{\texttt{https://github.com/SsnL/moco\_align\_uniform}}, where we also provide a model checkpoint (trained using $\lalign$ and $\lunif$) of $67.694\%$ validation \textrm{top1} accuracy.

\subsection{\bookcorpus with Quick-Thought Vectors Variants}
\paragraph{\bookcorpus details.} Since the original \bookcorpus dataset \citep{moviebook} is not distributed anymore, we use the unofficial code by \citet{sosuke2019bookcorpusgithub} to recreate our copy. Our copy ended up containing $52{,}799{,}513$ training sentences and $50{,}000$ validation sentences, compared to the original copy used by Quick-Thought Vectors \citep{logeswaran2018efficient}, which contains $45{,}786{,}400$ training sentences and $50{,}000$ validation sentences.

\paragraph{Quick-Thought Vectors with $\lalign$ and $\lunif$. } With Quick-Thought Vectors, the positive pairs are the neighboring sentences. At each optimization iteration, let \begin{itemize}
    \item $\{x_i\}_{i=1}^K$ be the $K$ \emph{consecutive} sentences forming this minibatch, where $K$ be the minibatch size,
    \item $f$ and $g$ be the two RNN sentence encoders.
\end{itemize} The original Quick-Thought Vectors \citep{logeswaran2018efficient} does not $l2$-normalize on encoder outputs during training the encoder. Here we describe the calculation of $\lcontr$, $\lalign$, and $\lunif$ for $l2$-normalized encoders, in our modified Quick-Thought Vectors method. Note that this does not affect evaluation since features are $l2$-normalized before using in downstream tasks, following the original Quick-Thought Vectors \citep{logeswaran2018efficient}. For a minibatch, these losses are calculated as following: \begin{itemize}
    \item $\lcontr$ with temperature:
    \begin{align*}
        & \frac{1}{K}~\mathtt{cross\_entropy}(\mathtt{softmax}(\{f(x_1)\T g(x_j)\}_j), \{0, 1, 0, \dots, 0\}) \\
        & \qquad + \frac{1}{K} \sum_{i=2}^{K-1} \mathtt{cross\_entropy}(\mathtt{softmax}(\{f(x_i)\T g(x_j)\}_j), \{\underbrace{0, \dots, 0}_{\text{$(i-2)$ $0$'s}}, \frac{1}{2}, 0, \frac{1}{2}, \underbrace{0, \dots, 0}_{\text{$(K-i-1)$ $0$'s}}\}) + \\
        & \qquad +\frac{1}{K}~\mathtt{cross\_entropy}(\mathtt{softmax}(\{f(x_K)\T g(x_j)\}_j), \{0, \dots, 1, 0\}).
    \end{align*}
    This is almost identical with the original contrastive loss used by Quick-Thought Vectors, except that this does not additionally manually masks out the entries $f(x_i)\T g(x_i)$ with zeros, which is unnecessary with $l2$-normalization.
    \item $\lalign$: The minibatch alignment loss is computed as disparity between features from the two encoders encoding neighboring sentences (assuming $K >= 2$): \begin{equation*}
        \frac{1}{K} \norm{f(x_1) - g(x_2)}_2^\alpha +
        \frac{1}{2K} \sum_{i=2}^{K-2} \left( \norm{f(x_{i-1}) - g(x_i)}_2^\alpha + \norm{f(x_{i}) - g(x_{i+1})}_2^\alpha \right) +
        \frac{1}{K} \norm{f(x_{K-1}) - g(x_K)}_2^\alpha.
    \end{equation*}
    \item $\lunif$: We combine the uniformity losses for each of $f$ and $g$ by summing them (instead of averaging since $f$ and $g$ are two different encoders): \begin{equation*}
        \frac{2}{K(K-1)} \sum_{i \neq j} e^{-t \norm{f(x_i) - f(x_j)}_2^2} + \frac{2}{K(K-1)} \sum_{i \neq j} e^{-t \norm{g(x_i) - g(x_j)}_2^2}.
    \end{equation*}
\end{itemize}

Our experiment settings below mostly follow the official implementation by \citet{logeswaran2018efficient}: \begin{itemize}
    \item Sentence encoder architecture is bi-directional Gated Recurrent Unit (GRU) \citep{cho-etal-2014-learning} with inputs from a $620$-dimensional word embedding trained jointly from scratch.
    \item We use Adam \citep{kingma2014adam} with $\beta_1=0.9, \beta_2=0.999, \epsilon=10^{-8}$, $400$ batch size, $0.0005$ constant learning rate, and $0.5$ gradient norm clipping.
    \item Optimization is done during $1$ epoch over the training data.
    \item For evaluation on a binary classification task, we freeze the encoder, and fit a logistic classifier with $l2$ regularization on the encoder outputs. A $10$-fold cross validation is performed to determine the regularization strength among $\{1, 2^{-1}, \dots, 2^{-8}\}$, following \citet{kiros2015skip} and \citet{logeswaran2018efficient}. The classifier is finally tested on the validation split.
\end{itemize}

\Cref{supp:tbl:bookcorpus-big} below describes the full specifications of all $108$ \bookcorpus encoders along with $6$ settings that lead to training instability (\ie, $\mathtt{NaN}$ occurring). These experiment results are visualized in main paper \Cref{fig:expr_scatter_bookcorpus}, showing a clear connection between representation quality and $\lalign$ \& $\lunif$ metrics. For the unnormalized encoders, the features are normalized before calculated $\lalign$ and $\lunif$ metrics, since they are nonetheless still normalized before being used in downstream tasks \citep{logeswaran2018efficient}.

\newpage

{
\newcommand{\scalem}[1]{\scalebox{0.62}{{\normalsize #1}}}
\newcommand{\scalemb}[1]{\scalebox{0.75}{{\normalsize #1}}}
\fontsize{6.5}{8}\selectfont
\centering
\newcommand{\mround}[1]{\round{#1}{4}}
\newcommand{\mroundprec}[1]{\round{#1}{2}\%}%
\renewcommand{\arraystretch}{1.2}%
\newcommand{\NA}{---}%
\renewcommand{\tabcolsep}{1.5pt}%
\setlength\LTleft{-1.32em}%
\begin{longtable}{r|c|c|c|c|c|c|c|c|c|c|c|c|c|c|c|c|}
\caption{Experiment specifications for all $304$ \stl encoders. We report the encoder representation quality measured by accuracy of linear and $k$-nearest neighbor ($k$-NN) with $k=5$ classifiers on either encoder outputs or \textrm{fc7} activations, via both a $5$-fold cross validation of the training set and the held out validation set.
\\
For encoder initialization, ``\texttt{rand}'' refers to standard network initialization, and symbols denote finetuning from a pretrained encoder, obtained via the experiment row marked with the same symbol.  Initial learning rates (LRs) are usually either fixed as $0.12$ or computed via a linear scaling ($0.12$ per $256$ batch size). Dimensionality (abbreviated as ``Dim.'') shows the ambient dimension of the output features, \ie, they live on the unit hypersphere of one less dimension. The last three rows show encoders that are used to initialize finetuning, but are not part of the $285$ encoders plotted in main paper Figure~3, due to their unusual batch size of $786$. Their accuracy and $\lalign$ \& $\lunif$ metrics follow the same trend shown in Figure~\ref{fig:expr_scatter_stl10}.  } \label{supp:tbl:stl10-big} \\
    \ltcline{~:----------------:}

    &
    \multicolumn{3}{c|}{Losses} &
    \multirow{3}{*}{Init.} &
    \multirow{3}{*}{Epochs} &
    \multirow{3}{*}{\shortstack{Batch\\Size}} &
    \multirow{3}{*}{Initial LR} &
    \multirow{3}{*}{Dim.} &
    \multicolumn{4}{c|}{Training Set 5-Fold Cross Val.\ Accuracy $\uparrow$} &
    \multicolumn{4}{c|}{Validation Set Accuracy $\uparrow$}
    \\
\end{longtable}
}
\vspace{2.5cm}

\newpage

{
\newcommand{\scalem}[1]{\hspace{2pt}\scalebox{0.72}{{\normalsize #1}}\hspace{2pt}}
\newcommand{\scalemb}[1]{\scalebox{0.82}{{\normalsize #1}}}
\fontsize{7}{8}\selectfont
\centering
\newcommand{\mround}[1]{\round{#1}{4}}
\newcommand{\mroundprec}[1]{\round{#1}{2}\%}
\renewcommand{\arraystretch}{1.2}
\newcommand{\NA}{---}
\renewcommand{\tabcolsep}{2.25pt}
\begin{longtable}{r|c|c|c|c|c|c|c|c|c|c|c|c|}
\caption{Experiment specifications for all $64$ \nyudepth encoders. We report the encoder representation quality measured by mean squared error (MSE) of a CNN depth predictor trained on \textrm{conv5} or \textrm{conv4} activations, via both a $5$-fold cross validation of the training set and the held out validation set.\\
All encoders in this table use standard network initialization (denoted as ``\texttt{rand}'').  Dimensionality (abbreviated as ``Dim.'') shows the ambient dimension of the output features, \ie, they live on the unit hypersphere of one less dimension.} \label{supp:tbl:nyudepth-big} \\
    \ltcline{~:------------:}
    &
    \multicolumn{3}{c|}{\multirow{2}{*}{Losses}} &
    \multirow{3}{*}{Init.} &
    \multirow{3}{*}{Epochs} &
    \multirow{3}{*}{\vspace{-2pt}\shortstack{Batch\\Size}} &
    \multirow{3}{*}{Initial LR} &
    \multirow{3}{*}{Dim.} &
    \multicolumn{2}{c|}{\multirow{2}{*}{\vspace{-2pt}\shortstack{Training Set 5-Fold\\Cross Val.\ MSE $\downarrow$}}} &
    \multicolumn{2}{c|}{\multirow{2}{*}{\vspace{-2pt}\shortstack{Validation Set\\MSE $\downarrow$}}}
    \\
    \ltcline{~:~~~:~:~:~:~:~:~~:~~:}
    & \multicolumn{3}{c|}{} & & & & & & \multicolumn{2}{c|}{} & \multicolumn{2}{c|}{}
    \\
    \ltcline{~:---:~:~:~:~:~:--:--:}
    &
    \scalemb{$\lcontr$} &
    \scalemb{$\lalign$} &
    \scalemb{$\lunif$} &
    & & & & &
    \hspace{6pt}\textrm{conv5}\hspace{6pt} &
    \hspace{6pt}\textrm{conv4}\hspace{6pt} &
    \hspace{6pt}\textrm{conv5}\hspace{6pt} &
    \hspace{6pt}\textrm{conv4}\hspace{6pt}
    \\
    \ltcline{~:------------:}
    \noalign{\vskip\doublerulesep
         \vskip-\arrayrulewidth}
    \ltcline{~:------------:}

 & \NA & \scalem{$0.5\cdot\lalignshort(\alpha\narroweq2)$} & \scalem{$\lunifshort(t\narroweq2)$} & \texttt{rand} & 400 & 128 & 0.06 & 128 & \mround{0.7405080556869507} & \mround{0.7979137301445007} & \mround{0.7378413677215576} & \mround{0.7969315648078918} \\
\end{longtable}
}

\newpage

{
\newcommand{\scalem}[1]{\hspace{2pt}\scalebox{0.72}{{\normalsize #1}}\hspace{2pt}}
\newcommand{\scalemb}[1]{\scalebox{0.82}{{\normalsize #1}}}
\fontsize{7}{8}\selectfont
\centering
\newcommand{\mround}[1]{\round{#1}{4}}
\newcommand{\mroundprec}[1]{\round{#1}{2}\%}
\renewcommand{\arraystretch}{1.2}
\newcommand{\NA}{---}
\renewcommand{\tabcolsep}{2.25pt}
\begin{longtable}{|c|c|c|c|c|c|c|c|c|c|c|c|c|c|}
\caption{Experiment specifications for all $45$ \imagenetsubset ResNet50 encoders trained using methods based on Momentum Contrast (MoCo) \citep{he2019momentum}. We report the encoder representation quality measured by accuracy of a linear classifier on penultimate layer activations, via both a $3$-fold cross validation of the training set and the held out validation set.\\
All encoders in this table use standard network initialization (denoted as ``\texttt{rand}'').  Dimensionality (abbreviated as ``Dim.'') shows the ambient dimension of the output features, \ie, they live on the unit hypersphere of one less dimension.\\
For $\lunif$, the ``Intra-batch'' column denotes whether $\lunif$ calculation includes pairwise distances within batch in addition to distances \wrt to the queue (\ie, Equation~\eqref{supp:eq:moco:lunif-queue-and-intra-batch} vs.~Equation~\eqref{supp:eq:moco:lunif-queue-only}).} \label{supp:tbl:imagenet100-big} \\

    \hline
    \multicolumn{4}{|c|}{\multirow{2}{*}{\vspace{-2pt}Losses}} &
    \multirow{4}{*}{Init.} &
    \multirow{4}{*}{Epochs} &
    \multirow{4}{*}{\shortstack{Batch\\Size}} &
    \multirow{4}{*}{\shortstack{Queue\\Size}} &
    \multirow{4}{*}{Initial LR} &
    \multirow{4}{*}{Dim.} &
    \multicolumn{2}{c|}{\multirow{2}{*}{\vspace{-2pt}\shortstack{Training Set 3-Fold\\Cross Val.\ Accuracy $\uparrow$}}} &
    \multicolumn{2}{c|}{\multirow{2}{*}{\vspace{-2pt}\shortstack{Validation Set\\Accuracy $\uparrow$}}}
    \\
    \multicolumn{4}{|c|}{} & & & & & & & \multicolumn{2}{c|}{} & \multicolumn{2}{c|}{}
    \\
\end{longtable}
}
\vspace{2.5cm}

\newpage

{
\newcommand{\scalem}[1]{\hspace{2pt}\scalebox{0.72}{{\normalsize #1}}\hspace{2pt}}
\newcommand{\scalemb}[1]{\scalebox{0.82}{{\normalsize #1}}}
\fontsize{7}{8}\selectfont
\centering
\newcommand{\mround}[1]{\round{#1}{4}}
\newcommand{\mroundprec}[1]{\round{#1}{2}\%}
\renewcommand{\arraystretch}{1.2}
\newcommand{\NA}{---}
\renewcommand{\tabcolsep}{2.25pt}
\begin{longtable}{|c|c|c|c|c|c|c|c|c|c|c|c|c|}
\caption{Experiment specifications for all $108$ \bookcorpus recurrent encoders trained using methods based on Quick-Thought Vectors \citep{logeswaran2018efficient}. We report the encoder representation quality measured by accuracy of logistic classifiers on encoder outputs for the Movie Review Sentence Polarity (\moviereview) and Customer Product Sentiment (\customerreview) binary classification tasks, via both a $5$-fold cross validation of the training set (of the downstream task) and the held out validation set (of the downstream task).\\
All encoders in this table use standard network initialization (denoted as ``\texttt{rand}'').  Dimensionality (abbreviated as ``Dim.'') shows the ambient dimension of the output features, \ie, features from $l2$-normalized encoders live on the unit hypersphere of one less dimension. Regardless of whether the encoder is $l2$-normalized (indicated in ``Normalization'' column), the features are always normalized before being used for downstream tasks, following \citet{logeswaran2018efficient}. \\
The only unnormalized encoder is obtained using the unmodified Quick-Thought Vectors algorithm. $6$ configurations that suffer from training instability (\ie, $\mathtt{NaN}$ occurring) are also reported. } \label{supp:tbl:bookcorpus-big} \\

    \hline
    \multicolumn{3}{|c|}{\multirow{2}{*}{\vspace{-2pt}Losses}} &
    \multirow{3}{*}{Normalization} &
    \multirow{3}{*}{Init.} &
    \multirow{3}{*}{Epochs} &
    \multirow{3}{*}{\shortstack{Batch\\Size}} &
    \multirow{3}{*}{Initial LR} &
    \multirow{3}{*}{Dim.} &
    \multicolumn{2}{c|}{\multirow{2}{*}{\vspace{-2pt}\shortstack{Training Set 5-Fold\\Cross Val.\ Accuracy $\uparrow$}}} &
    \multicolumn{2}{c|}{\multirow{2}{*}{\vspace{-2pt}\shortstack{Validation Set\\Accuracy $\uparrow$}}}
    \\
    \multicolumn{3}{|c|}{} & & & & & & & \multicolumn{2}{c|}{} & \multicolumn{2}{c|}{}
    \\
\end{longtable}
}

\end{document}